\documentclass[accepted]{uai2025}
\usepackage{natbib} % has a nice set of citation styles and commands
\usepackage{microtype}
\usepackage{graphicx}
\usepackage{subfigure}
\usepackage{booktabs} % for professional tables

\usepackage{hyperref}

\usepackage{algorithmic}
\usepackage{algorithm}
\usepackage[american]{babel}
\usepackage{amsmath}
\usepackage{amssymb}
\usepackage{amsfonts}
\usepackage{mathtools}
\usepackage{amsthm}

\usepackage[capitalize,noabbrev]{cleveref}

\theoremstyle{plain}
\newtheorem{theorem}{Theorem}[section]

\newtheorem{lemma}[theorem]{Lemma}

\theoremstyle{definition}
\newtheorem{example}{Example}
\newtheorem{definition}[theorem]{Definition}

\theoremstyle{remark}

\usepackage[textsize=tiny]{todonotes}

\makeatletter
\DeclareFontFamily{U}{tipa}{}
\DeclareFontShape{U}{tipa}{m}{n}{<->tipa10}{}
\newcommand{\arc@char}{{\usefont{U}{tipa}{m}{n}\symbol{62}}}%

\newcommand{\arc}[1]{\mathpalette\arc@arc{#1}}

\newcommand{\arc@arc}[2]{%
  \sbox0{$\m@th#1#2$}%
  \vbox{
    \hbox{\resizebox{\wd0}{\height}{\arc@char}}
    \nointerlineskip
    \box0
  }%
}
\makeatother

\newcommand{\argmin}{\mathop{\rm argmin}\limits}

\newcommand{\boldtheta}{{\boldsymbol{\theta}}}
\newcommand{\bolddelta}{{\boldsymbol{\delta}}}

\newcommand{\boldK}{{\boldsymbol{K}}}

\newcommand{\boldTheta}{{\boldsymbol{\Theta}}}
\newcommand{\boldf}{{\boldsymbol{f}}}

\newcommand{\boldGamma}{{\boldsymbol{\Gamma}}}

\newcommand{\boldbeta}{{\boldsymbol{\beta}}}

\newcommand{\mathbbR}{\mathbb{R}}

\newcommand{\boldh}{{\boldsymbol{h}}}

\newcommand{\boldx}{{\boldsymbol{x}}}

\newcommand{\boldg}{{\boldsymbol{g}}}

\newcommand{\boldw}{{\boldsymbol{w}}}

\newcommand{\boldphi}{{\boldsymbol{\phi}}}
\newcommand{\boldzero}{{\boldsymbol{0}}}

\newcommand{\boldX}{{\boldsymbol{X}}}

\newcommand{\boldy}{{\boldsymbol{y}}}

\newcommand{\boldI}{{\boldsymbol{I}}}

\newcommand{\mathbbE}{\mathbb{E}}

\newcommand{\vertiii}[1]{{\left\vert\kern-0.25ex\left\vert\kern-0.25ex\left\vert #1 
    \right\vert\kern-0.25ex\right\vert\kern-0.25ex\right\vert}}
\newcommand{\bh}{\boldsymbol{h}}

\newcommand{\bw}{\boldsymbol{w}}

\title{Guiding Time-Varying Generative Models with Natural Gradients \\ on Exponential Family Manifold}
\author[1]{\href{mailto:song.liu@bristol.ac.uk}{Song Liu}}
\author[2]{\href{mailto:leyang.wang.24@ucl.ac.uk}{Leyang Wang}}
\author[1]{\href{mailto:yakun.wang@bristol.ac.uk}{Yakun Wang}}
\affil[1]{%
    School of Mathematics\\
    University of Bristol\\
    United Kingdom
}
\affil[2]{%
    Department of Computer Science\\
    University College London\\
    United Kingdom
}
\begin{document}
\maketitle

\begin{abstract}
Optimising probabilistic models is a well-studied field in statistics. However, its connection with the training of generative models 
remains largely under-explored. In this paper, we show that the evolution of time-varying generative models can be projected onto an exponential family manifold, naturally creating a link between the parameters of a generative model and those of a probabilistic model. We then train the generative model by moving its projection on the manifold according to the natural gradient descent scheme. 
This approach also allows us to efficiently approximate the natural gradient of the KL divergence without relying on MCMC for intractable models.  Furthermore, we propose particle versions of the algorithm, which feature closed-form update rules for any parametric model within the exponential family. Through toy and real-world experiments, we validate the effectiveness of the proposed algorithms. The code of the proposed algorithms can be found at \url{https://github.com/anewgithubname/iNGD}.
% Probabilistic models are widely established in many fields of statistics. However, their integration with the training of generative models remains largely under-explored. In this paper, we show that the evolution of time-varying generative models can be projected onto an exponential family manifold, naturally creating a link between the parameters of a generative model and those of a probabilistic model. We then train the generative model by moving its projection on the manifold. This projection also allows us to apply established optimization methods, such as natural gradient descent, to efficiently navigate the exponential family manifold. Thanks to the time-varying generative model, the likelihood gradient and Fisher information matrix can be approximated efficiently without relying on MCMC sampling. Furthermore, we propose a particle-based version of the algorithm, which features a closed-form update rule for any parametric model within the exponential family. Through toy and real-world experiments, we validate the effectiveness of the proposed algorithms.
% %First order optimization methods often require careful hyper-parameters tuning and exihibit low performance when low condition number. Though natural gradient descent is beneficial to 

% %In this paper, we perform natural gradient descent approximately using a time-varying generative model. 
\end{abstract}

\section{Introduction}
Modern generative models \citep{goodfellow2014generative,ho2020denoising,song2021scorebased} have become indispensable tools in machine learning, achieving remarkable success in applications \citep{rombach2022high,gu2022vector,li2019neural,tan2024naturalspeech}.
% such as image generation \citep{rombach2022high,gu2022vector}, text generation \citep{openai2024gpt4technicalreport,deepseekai2025deepseekr1}, and speech synthesis \citep{li2019neural,tan2024naturalspeech}. 
% Despite their widespread success, a significant limitation of many contemporary generative models is their lack of interpretability \citep{ross2021evaluating}. 
These generative models are neural networks that transform a latent variable into a higher-dimensional sample. They overcome \emph{classic restrictions imposed on probability density models}, such as positivity, normalisation, and encoding of conditional independence via factorisation; thus, they can be designed freely to capture complex, intricate patterns from the high-dimensional data. 
% In particular, without the normalization constraint, these models could be trained by optimizing empirical objectives. 
% Such flexibility enables them to model complex, high-dimensional data distributions, allows neural networks to capture intricate patterns and dependencies in data, 
% making them powerful tools for generative tasks.
% This freedom enables neural network designs 
Although these models produce highly realistic outputs, they can be hard to train. Training them requires massive amounts of data and maintains a delicate balance between ``generators'' and ``discriminators'' \citep{goodfellow2014generative,arjovsky2017wasserstein,wang2022diffusion} or building effective  ``bridges'' between the reference dataset and target dataset \citep{song2021scorebased,lipman2023flow,liu2023flow,bortoli2021diffusion} which can be difficult to design. 

% to provide probabilistic insights  into the underlying mechanisms of data generation, and are difficult to optimize. This limiation 

% Despite these widespread success, it remains unclear how these generative models are associated with conventional, well-established probalistic models. 
% This lack of connection makes it challenging to analyze their behaviors, understand the model insights \citep{ross2021evaluating} and optimize these models as our conventional statistical inference approach breaks down.

% While these models produce highly expressive latent representations and realistic outputs, they often fail to provide meaningful insights into the underlying mechanisms of data generation. This opacity makes it challenging to analyse their behaviour and limits their applicability in critical domains where transparency is essential, such as healthcare, finance, and scientific research.

Parametric probabilistic models, particularly those in the 
exponential family \citep{casella2024statistical, wainwright2008graphical}, 
play a central role in modern statistical inference, and have
% well-developed training algorithms and a complete 
well-established theoretical framework and training algorithms. 
These models, characterised by their sufficient statistics and natural parameters, define a \emph{probability distribution manifold}, 
% provide a structured framework for encoding the conditional independence of data. 
the geometric structure of which inspired efficient optimisation methods such as natural gradient descent (NGD) \cite{amari1998natural}, ensuring stable and efficient parameter estimation \citep{amari2000methods}. However, while exponential family models are versatile \emph{in theory}, their use may be limited in practice: the hand-crafted sufficient statistic may fail to capture complex relationships in data; more flexible sufficient statistics (such as neural nets) result in intractable likelihoods. Thus, parameter estimation that requires a likelihood, such as NGD, cannot be easily applied to fit the model. 
% However, while exponential family models offer transparency and mathematical rigour, their reliance on fixed parametric forms limits their capacity to handle complex, high-dimensional data distributions effectively \citep{kakade2010learning}. 

We aim to unlock the power of modern generative models through the principled training of probabilistic models.

Our work is inspired by two distinct research directions developed in recent years: time-varying generative models \citep{ho2020denoising,song2021scorebased,liu2023flow,lipman2023flow} and Time Score Matching (TSM) \citep{choi2022density}.
Time-varying generative models generate samples progressively, evolving them over time until they match the target distribution. Meanwhile, TSM learns the instantaneous change of a time-varying distribution from data. Recent work demonstrates that TSM can ``project'' temporal variations of a dataset onto the parameter space of exponential family distributions \citep{williams2024high}.

% Moreover, recent works \cite{choi2022density,williams2024high} show that the instantaeous changes a time-varying dataset (such as a time-series) could be estimated via a technique called Time Score Matching (TSM), using a time-varying exponential family model.

The core idea of this paper is to evolve a generative model such that its \emph{projected trajectory} on an exponential family manifold aligns with the trajectory induced by NGD. 
This way, we get the expressiveness of a modern generative model, and the training efficiency of NGD. The exponential family model acts as a guiding framework for the generative model throughout the training process.
% Generative models can capture highly complex data distributions but often lack transparency regarding how specific patterns are generated or influenced. On the other hand, exponential family models explicitly represent relationships between variables, allowing for a deeper understanding of the data generation process. 
% By combining the strengths of these two paradigms, our research seeks to guide generative models using the structure and optimisation properties of exponential family manifolds, enabling models that are both interpretable and expressive.
% In this paper, we propose a novel framework that aligns the evolution of time-varying generative models with the natural gradient dynamics on exponential family manifolds. 
An illustration of this idea is in \cref{fig.notions}. 
We align the projected changes of the generative model with the NGD update on a parametric manifold (shrinking the length of the red dotted line in \cref{fig.notions}). 
% First, we apply TSM to project the temporal evolution of the generative model onto the manifold. Second, we match this projected evolution to the NGD update by updating the parameters of the generative model. This process repeats until convergence. 

Although our technique applies to all time-varying generative models, we focus on drift-based generative models, where samples are iteratively perturbed by vector-valued functions. We develop two NGD-guided drift-based generative approaches: kernel NGD and neural tangent kernel NGD, both of which admit closed-form sample updates.

% that in several cases, the particle updates could be obtained in closed forms. 
% Finally, we introduce the Stein exponential family, which enables tractable minimization of the forward KL divergence in a Bayesian inference setting. We demonstrate that this family can also serve as an effective guiding framework for training generative models efficiently in variational inference tasks.
% This alignment ensures that the generative model evolves in a manner consistent with the geometric structure of the manifold. 

% Our approach introduces two key innovations: (1) a method for aligning the updates of generative models with the geometry of exponential family manifolds, and (2) a particle-based variant of the algorithm that features a closed-form update rule for any parametric model within the exponential family.

\begin{figure}
    \centering
    \includegraphics[width=.45\textwidth]{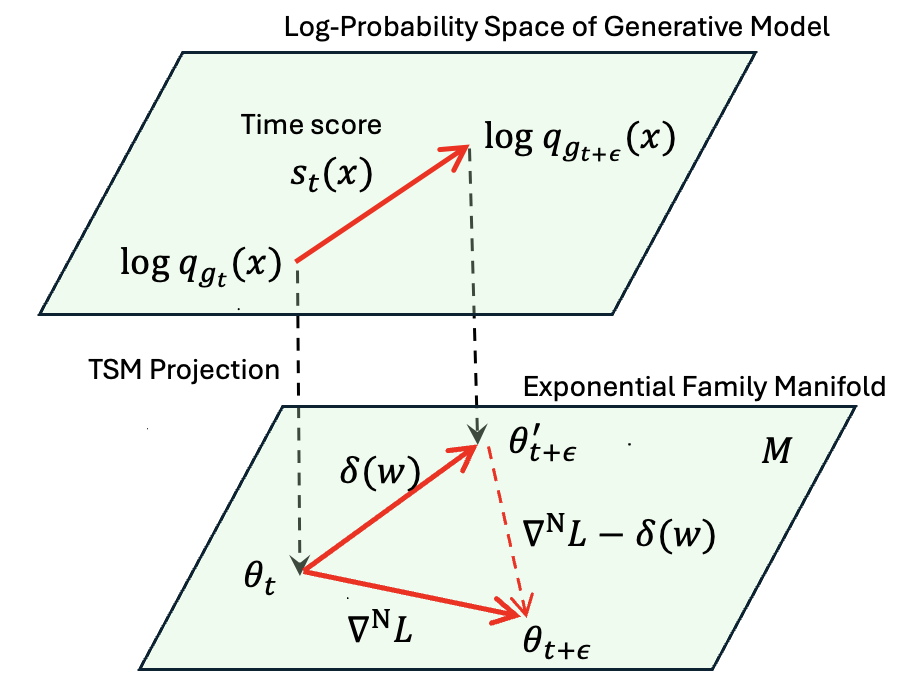}
    \caption{Technical notions illustrated. Matching the projected generative model change $\bolddelta(\boldw)$ to the NGD meaning reducing the length of the red dotted line. Symbols are defined in Section \ref{sec.background} and \ref{sec.prop}. }
    \label{fig.notions}
\end{figure}
% \todo{This plot is potentially misleading, change later.}

\section{Background}
\label{sec.background}
In recent years, it has been recognised that there are many similarities between sampling and optimisation \citep{wibisono2018sampling,arbel2019maximum,sinho2024logconcave,cheng2018underdamped,he2024training}. Encouraged by these results, we ask: 
% Can we train a generative model using Natural Gradient Descent (NGD)? 
% This research is motivated by the question: How to train a generative model using natural gradient descent? 
% At a glance, this question doens't make sense: Natural gradient descent is an optimization method for probabilistic models and a generative model is not a probabilistic model by definition. However, in recent years, it has been recognized that there are many the similarities between sampling and optimization. Encouraged by these results, we investigate the possibility of training a generative model using natural gradient descent. 
Given a time-varying generative model \( g_t \), where \( t \) is time, can we move its output distribution \( q_{g_t} \) toward an optimal distribution \( q^* \) using NGD?  In literature, generative models are sometimes referred to as implicit models and generative distribution \( q_{g_t} \) may not have a parametric density\footnote{Some generative models, such as normalizing flows \citep{rezende2015variational}, neural ODEs \citep{chen2018neural}, admit explicit parametric forms. }, thus optimisation designed for parametric densities cannot be directly applied. 

% In this paper, we propose an alternative: rather than directly optimizing \( q_{g_t} \), we move its projection onto a chosen manifold of exponential family distributions toward optimality. We select this exponential family to be sufficiently expressive, aiming to approximate the full range of distributions that the generative model can produce.

Before stating our solutions to this problem, we first introduce a few ideas that are essential for developing the proposed algorithm. 

\textbf{Notations:} 
Vectors (\(\boldx\)) and matrices (\(\boldX\)) are denoted in bold. $\mathrm{dim}(\boldx)$ denotes the dimension of vector $\boldx$. 
Random variables ($X$) are non-bold, capital letters. 
% with the inner product of \(\boldx, \boldy \in \mathbb{R}^d\) written as \(\langle \boldx, \boldy \rangle = \boldx^\top \boldy = \sum_{k=1}^d x_k y_k\). 
Euclidean norms are denoted as \(\|\cdot\|\), and for elements in a Hilbert space, the Hilbert norm is written as \(\|\cdot\|_{\mathcal{H}}\). 
% A probability density in the exponential family is given by \(q(\boldx; \boldtheta)\), where \(\boldtheta \in \mathbb{R}^d\) is the natural parameter, \(T(\boldx)\) is the sufficient statistic, and \(A(\boldtheta)\) is the log-normalisation constant. 
% For time-varying distributions, \(q_t(\boldx) = q(\boldx; \boldtheta(t))\), with \(\boldtheta(t)\) evolving over time, and \(\partial_t \log q_t\) depends on \(\partial_t \boldtheta(t)\) and \(T(\boldx)\). 
Expectations and covariances under \(q\) are denoted as \(\mathbb{E}_{q}[\cdot]\) and \(\mathrm{Cov}_{q}[\cdot]\), respectively. 
$\nabla f(\boldx) = [\partial_1 f(\boldx), \partial_2 f(\boldx), \cdots]^\top$. 
$\nabla f(\boldx_0)$ means the gradient of $f$ evaluated at $\boldx_0$. Suppose $\boldf: \mathbbR^m \to \mathbbR^n, $
$\nabla \boldf(\boldx)$ denotes the Jacobian of $\boldf$ and is a matrix of size $n \times m$. 

% \section{Background}
% In this section, we introduce a few key concepts that are essential for understanding the proposed algorithm.
\subsection{Time-varying Exponential Family}
\begin{definition}(See, e.g., Section 3.4, \citet{casella2024statistical})
\label{def.exp.fam}
A distribution is a member of the \textbf{exponential family} with the sufficient statistic \(T\) if and only if its density function \(q(\boldx; \boldtheta)\) can be expressed as:
\begin{align*}
    q(\boldx; \boldtheta) := \exp\big(\langle \boldtheta, T(\boldx) \rangle - A(\boldtheta)\big),
\end{align*}
where \(\boldtheta \in \mathbbR^d\) is the natural parameter, and \(A(\boldtheta)\) is the log-normalisation function, defined as:
\begin{align*}
    A(\boldtheta) := \log \int \exp\big(\langle \boldtheta, T(\boldx) \rangle\big) \, d\boldx.
\end{align*}
\end{definition}
This family includes many known distributions, such as Gaussian, Gamma and Exponential distributions. We can also choose $T$ to be infinite-dimensional to enable more flexible modelling \citep{sriperumbudur2017density,arbel2018kernel}. 

\begin{definition}
The parametric manifold of the exponential family distributions with a sufficient statistic $T$ is:
\begin{align*}
    \mathcal{M}(T) := \left\{\boldtheta \in \mathbb{R}^d \;\middle|\; \int q(\boldx; \boldtheta) \, d\boldx = 1 \right\},
\end{align*}
where $q(\boldx; \boldtheta)$ depends on $T$ as in Definition \ref{def.exp.fam}. 
\end{definition}

Although $\mathcal{M}(T)$ is a parametric manifold, with a slight abuse of the notation, we denote $q_\boldtheta \in \mathcal{M}(T)$ to indicate that $q_\boldtheta$ is a member of exponential family. 

\textbf{Time-varying exponential family distributions} refers to exponential family distributions whose natural parameter is a continuous function of time, i.e., \(q_{\boldtheta_t} = q(\boldx; \boldtheta(t))\). 
% We can see that given a sufficient statistic $T$, $\boldtheta(t)$ is a curve on $\mathcal{M}(T)$. 
% Given a continous time-varying parameter \(\boldtheta(t)\), we can define a sequence of distributions \(q_t = q(\boldx; \boldtheta(t))\) and 
Moreover, 
the time derivative of \(\log q_t\) is:
\begin{align*}
    \partial_t \log q_{\boldtheta_t} = \langle \partial_t \boldtheta(t), T(\boldx) \rangle - \partial_t A(\boldtheta(t)).
\end{align*}
% Below, we refer to the time derivative $\partial_t \log q_{\boldtheta_t}$ as the ``time score'', denoted as $s_t$. 

Equivalently, this can be expressed as the inner product of the natural parameter’s rate of change and the centred sufficient statistic (Proposition 3.1 in \citet{williams2024high}):
\begin{align}
\label{eq:partial_t.logqt}
    \partial_t \log q_{\boldtheta_t} = \langle \partial_t \boldtheta(t), T(\boldx) - \mathbb{E}_{q_{\boldtheta_t}}[T(\boldx)] \rangle.
\end{align}
By definition, the time-varying process $\boldtheta(t)$ is a curve on $\mathcal{M}(T)$ and $\partial_t \boldtheta(t)$ is the tangent vector of such curve. 

\subsection{Natural Gradient Descent}
The natural gradient descent is an effective optimisation technique for parametric probabilistic models  \citep{amari1998natural,amari2000methods} and has been widely applied in various fields \citep{bernacchia2018exact,chenteng}.
% including neural networks' training 
% generative models \citep{shen2020sinkhorn,lin2021wasserstein},  
% and PINN training \citep{muller2023achieving,chenteng}
\begin{definition}
    The natural gradient of a loss function $\mathcal{L}(\boldtheta)$ is the gradient of the loss function scaled by the inverse Fisher information matrix $\mathcal{F}$. 
    \begin{align*}
        \nabla^N \mathcal{L}(\boldtheta) := \mathcal{F}^{-1} \nabla \mathcal{L}(\boldtheta).
    \end{align*}
    Suppose $q_\boldtheta \in \mathcal{M}(T)$, the Fisher information matrix is 
    \[
    \mathcal{F} = \nabla^2_\boldtheta \log q(\boldx; \boldtheta) = \mathrm{Cov}_{q_{\boldtheta}}[T(\boldx)],
    \]
    describing the curvature of $\mathcal{M}(T)$ neighbouring $\boldtheta$.
\end{definition}
% The continuous-time limit of NGD, known as Natural Gradient Flow, describes a time-varying distribution \( q_{\boldtheta_t} \), whose parameters trace a trajectory on \( \mathcal{M}(T) \) following the steepest descent direction to minimize \( \mathcal{L} \). 
Although methods described in this paper are generic for different $\mathcal{L}$, we focus on the KL divergence in this paper.

\begin{example}
    Suppose $\mathcal{L}$ is the KL divergence from $p$ to $q$.  
    \begin{align*}
        \mathcal{L} = \mathrm{KL}[p, q] = \mathbb{E}_{p}[\log p(\boldx) - \log q(\boldx)]. 
    \end{align*}
    Suppose $q$ contains parameters $\boldtheta$ and $q_\boldtheta \in \mathcal{M}(T)$, 
    then the natural gradient of $\mathcal{L}$ is given by
    \begin{align}
        \label{eq.ngd.kl}
        \nabla^N {\mathrm{KL}} &:= - \mathcal{F}^{-1} \mathbb{E}_{p}[\nabla_{\boldtheta} \log q(\boldx; \boldtheta)] \notag \\
        &= - \mathrm{Cov}_{q_{\boldtheta}}[T(\boldx)]^{-1} \left\{\mathbb{E}_{p}[T(\boldx)] - \mathbb{E}_{q_{\boldtheta}}[T(\boldx)]\right\}. 
    \end{align}
\end{example}
Except for certain specific choices of \( T \), \( \nabla^N {\mathrm{KL}} \) does not admit a closed-form expression, as neither \( \mathbb{E}_{q_{\boldtheta}}[T(\boldx)] \) nor \( \mathrm{Cov}_{q_{\boldtheta}}[T(\boldx)] \) can be expressed in closed form for a general $T$. If we could sample from \( q_{\boldtheta} \), these expectations and covariances could be approximated using Monte Carlo methods. However, generating samples from a complex distribution \( q_{\boldtheta} \) itself remains a challenging problem.  

\subsection{Time-varying Generative Model}
In recent years, there has been a growing trend of designing generative models as functions of time, in contrast to the classic generative models where the sample generating mechanism is independent of time. 
For example, the diffusion generative model \citep{song2021scorebased} can be interpreted as the solution to a Stochastic Differential Equation at time \( t \),  
with initial samples drawn from a reference distribution.  
Similarly, the rectified flow generative model \citep{liu2023flow} can be viewed as the solution to an Ordinary Differential Equation at time \( t \).
\begin{definition}
    A \textbf{generative model} is a data-generating mechanism defined as \(X = \boldg(Z; \boldw)\), where \( Z \) is a random variable sampled from a latent distribution \( p_Z \). \(\boldw \in \mathcal{W}\) are parameters of the generative model $\boldg$. 
    A \textbf{time-varying generative model} is  defined as \( X_t = \boldg(Z, t; \boldw) \), i.e., a generative model whose output depends explicitly on $t$. 
\end{definition}

\subsection{Time Score Matching}
In applications like time-series analysis, characterising the change of the data generative distribution is an essential task, and recently, a method measures the change of distribution over time was proposed \citep{choi2022density}. 
% Now we introduce how to measure the change of distributions over time using samples.  
% time score measures the rate of change of the distribution over time. 

\begin{definition}
\textbf{Time score} is the time derivative of the log density of a time-varying distribution. Given a time-varying distribution $q_t$, its time score is $s_t := \partial_t (\log q_t)$. 
\end{definition}

% Rather than attempting to directly compute the ratio between the original distributions, which can be ill-defined or otherwise difficult, TSM introduces an infinite continuum of intermediate “bridge” distributions. Each pair of neighboring distributions in this continuum exhibits a smaller distributional discrepancy, facilitating more accurate density ratio estimates at each step. These intermediate estimates are then aggregated to obtain the overall statistical divergence or density ratio.\\

Given a parametric time score model \({v}(\boldx;t)\) and a time-varying sample $X_t \sim q_t$, the time score can be learned by 
% By modeling the instantaneous rate of change in \(\log q_t(\boldx)\) with a parametric score function \({s}(\boldx;t)\), thus 
the \textbf{Time Score Matching (TSM)} which minimises the objective:
$
    \int 
    \mathbb{E}\Bigl[
        \lambda(t)\,\bigl\|
        s_t(X_t)
        -
        {v} \bigl(X_t;t\bigr)
        \bigr\|^2
    \Bigr]\,
    \mathrm{d}t,
    % \label{eq:deriv:obj1}
$
where \(\lambda(t)\) is a weighting function. 
% satisfying specified boundary conditions. 

% The time score plays an important role in the density ratio estimation as $\log (q_{t_1}/q_{t_0}) = \int_{t_0}^{t_1} s_t \mathrm{d}t$. \citeauthor{choi2022density} leverages this fact and 
% estimates the time score at intermediate distributions and aggregates them to obtain the overall ratio.

A recent work shows, 
for the special case where \({q_t}(\mathbf{x}) \in \mathcal{M}(T)\), the time differential natural parameter \(\partial_t\boldsymbol{\theta}\) can directly learned by TSM \citep{williams2024high} and applied it to learning time-varying graphical models. 
% without learning $\boldsymbol{\theta}(t)$ at each time-point and then differentiate, thus streamlining the estimation process.
%  and leveraging the inherent structure of the exponential family.

% We also want to generates samples from such a approximate distribution $q_T$. 

% Let the trajectory of FRGF be $p_t$. We want to find a sequence of distributions $q_t$ that approximates $p_t$. Moreover, we want to find a generative sequence $g_t$ that generates samples from $q_t$ for all $t$.

% We assume that our generator $g$ at time $t=0$ generates samples from an exponential family distribution that is easy to sample from. 

\section{Proposed Algorithm: Evolution Projection}
\label{sec.prop}
\subsection{Problem Formulation}
Suppose we only have access to a target distribution \( p \) through samples \( Y \sim p \) and a latent variable $Z$. Our goal is to find a time-varying generative model \( g_t \) that progressively approximates the target distribution as \( t \to \infty \). Informally speaking, we seek a model such that \( Y \overset{d}{\approx}  g_\infty(Z) \).

We want to avoid training generative models using GANs or diffusion models, as these models require a large number of samples and are hard to train. Instead, given a loss function \(\mathcal{L}(p, q_{\boldtheta_t})\) that measures the difference between \(p\) and a time-varying probabilistic model \(q_{\boldtheta_t}\), we aim to guide the training of the generative model by minimising the loss \(\mathcal{L}\) of the probabilistic model \(q_{\boldtheta_t}\) over time.

\textbf{The key idea} of this paper is to align the evolution of the generative model with \(\nabla^N \mathcal{L}(\boldtheta)\) on the manifold \(\mathcal{M}(T)\). 
Consequently, minimising the loss for \( q_\boldtheta \) simultaneously drives the generative model toward matching \( p \).

% Thus, as we minimize the loss for \(q_\boldtheta\), we simultaneously evolve the generative model to match \(p\).

To achieve this, we first need to establish a direct correspondence between the instantaneous change in the generative model and the parametric update on \(\mathcal{M}(T)\).

%We use $\|\cdot\|$ to denote a norm on $\mathbb{R}^d$; the inner product is denoted as $\langle \bu,\bv\rangle = \bu^\top\bv = \sum_{k=1}^d u_k v_k$. We write $\partial_t$ as $\frac{\partial}{\partial t}$
% \begin{definition}
%     $r :\mathcal{W} \to \mathcal{M}$ is a parameterization that maps a weight vector in $\mathcal{W}$ onto the manifold $\mathcal{M}$. An open subset $U \subset \mathcal{M}$ is fully parameterized by $r$ if and only if for all $\boldtheta \in \mathcal{M}$, there exists a $\boldw \in \mathcal{W}$ such that $r(\boldw) = \boldtheta$.
% \end{definition}

% Once we have the parameterization, we can match the time-variation of the generator network to  the desired evolution of $\boldtheta$ under the natural gradient descent regime. 

% It may look difficult to come up with parameterization and establish the connection. However, the key insight is 

% Suppose $q_{g}$ is the distribution generated by $g$. 
% Let $q_{g_t}$ be the density of generated sample $X_t$. 
\subsection{Projecting the Change}
\label{sec.proj.change}

Denote the sample of a time-varying generative model $g_t$ as $X_t \sim q_{g_t}$. We can measure the instantaneous change of the generative model via the its time score $s_t := \partial_t (\log q_{g_t})$. 
% which represents the time-variation of the generative model in the probability space at a time $t$.  

% Suppose $X_t$ is the output of a time-varying generative model $g_t$. 
% We also have a time-varying exponential family model $q_{\boldtheta_t}$ with sufficient statistic $T$. 
At a fixed time $t_0$, we can ``project'' the time score $s_t$ 
onto the manifold $\mathcal{M}(T)$
% of $q_t$ 
by minimising the squared difference between $s_t$ and the time score of an exponential family distribution, $\partial_t (\log q_{\boldtheta_t}), q_{\boldtheta_t} \in \mathcal{M}(T)$, 
% following least squares objective with respect to $\bolddelta_{t_0}$:
\begin{align}
    \label{eq.tsm}
    &\int \lambda_{t_0}(t) \mathbbE \left[\left(s_t(X_t) - \partial_t (\log q_{\boldtheta_t})(X_t) \right)^2\right] \mathrm{d}t, \notag \\
    = &\int \lambda_{t_0}(t) \mathbbE \left[\left(s_t(X_t) - \langle \partial_t \boldtheta(t), T(X_t) - \mathbbE [T(X_t')] \rangle\right)^2\right] \mathrm{d}t, 
\end{align}
where $\lambda_{t_0} = \exp(-(t-t_0)^2/\sigma^2)$ and $X_t'$ is an independent copy of $X_t$. $\sigma$ is a hyperparameter fixed in advance. The second line is due to \cref{eq:partial_t.logqt}. The integration is over the entire real line. 

% For brevity, we shorten $\lambda_{t_0}(t)$ as $\lambda(t)$ from now on. 
Introducing a linear-in-time model $\boldtheta(t; \bolddelta) = t \bolddelta$, the above objective becomes 
\begin{align}
    \label{eq.locallinear.obj}
    J(\bolddelta) := \int \lambda_{t_0}(t) \mathbbE \left[\left(s_t(X_t) - \langle \bolddelta, T(X_t) - \mathbbE [T(X_t')] \rangle\right)^2\right] \mathrm{d}t, 
\end{align}
which depends on the parameter $\bolddelta$. 
We use linear-in-time model to simplify the optimisation problem. More model choices of $\boldtheta(t)$ could be found in \citep{williams2024high}. 
\eqref{eq.locallinear.obj} can be viewed as a local regression at the fixed time point $t_0$:  $\lambda$ is a time smoothing kernel, \cref{eq.locallinear.obj} finds the best score model that approximates  $s_t(\boldx)$ at $t_0$. 
% and a 
% linear model $\langle \bolddelta, T(\boldx) - \mathbbE [T(X_t')] \rangle$ 
% with a time smoothing kernel $\lambda$.  

The minimiser to the above objective is a vector in $\mathbbR^d$ that best describes the instantaneous change in the generative model. Moreover, we can show that 
\begin{theorem}
\label{thm:delta3.1}
    Let $\bolddelta_{t_0} := \argmin J(\bolddelta)$, then 
    % Then $\bolddelta_{t_0}$ is unique if the Fisher information matrix $\mathcal{F}_t = \mathrm{Cov}[T(X_t)]$ is invertible. 
    \begin{align*}
        &\bolddelta_{t_0} =\nonumber \\&  - \left(\int \lambda_{t_0}(t) \mathrm{Cov}[T(X_t)]
        \mathrm{d}t\right)^{-1} \int \partial_t \lambda_{t_0}(t) \mathbbE \left[ T(X_t) \right] \mathrm{d} t. 
    \end{align*} 
\end{theorem} 
The proof can be found in \cref{app:3.1}. 
Since $X_t = g(Z, t; \boldw)$, we can express $\bolddelta_{t_0}$ as a function of $\boldw$ using the reparameterization trick \citep{Kingma2014}. 
\begin{align}
\label{eq.deltasol}
    \bolddelta_{t_0}(\boldw) &=  - C^{-1} \int \partial_t \lambda_{t_0}(t) \mathbbE \left[ T(g(Z, t; \boldw)) \right]\mathrm{d}t, \notag \\
    C &= \int \lambda_{t_0}(t) \mathrm{Cov}[T(g(Z, t; \boldw))]
    \mathrm{d}t.  
\end{align}
This expression allows us to approximate $\mathbbE[\cdot]$ and $\mathrm{Cov}[\cdot]$ using samples of $Z$. 
We illustrate 
the projection process in Figure \ref{fig.notions} where the downward dotted arrow represents the projection by TSM. 

% From now on, we refer to $\bolddelta_{t_0}(\boldw)$ as the projection of $\boldw$ at $t_0$. 

\paragraph{Remarks:}
% Although the mapping from the time score to the exponential family manifold is unique, the opposite is not necessarily true. 
Different time scores can map to the same projection $\bolddelta_{t_0}$, particularly if the sufficient statistic $T$ is restrictive. 
However, if $T$ is chosen to make the exponential family is expressive enough, we expect that such information collapse can be avoided. The rigorous proof of this claim is left as a future work. On a positive note, TSM objective $\eqref{eq.tsm}$ involves no normalising terms, consequently, we are free to use any expressive choice of $T$ (e.g. neural networks, or RKHS kernels \citep{sriperumbudur2017density}).  

The hyperparameter $\sigma$ in \cref{eq.tsm} will introduce additional biases to the estimation, similar to how a non-zero kernel bandwidth in local regression introduces biases to the estimate. In experiments, we observe that reasonable $\sigma$ choices (e.g., $0.1$) work well. Moreover, in \cref{sec.ngd.drift}, we show how to obtain an unbiased estimator for a special type of time-varying generative models.  

% It can be seen that if the generator is flexible enough, the projection of the generator onto the parameter space will cover a large subspace of the manifold. We assume that the generator is expressive enough so that its projection at least covers the neighborhood of $\boldtheta_{t_0}$. 

% \begin{assumption}
%     $\forall t$, the neighborhood of $\boldtheta(t)$ is fully parameterized by $r$ established using the above procedure. 
% \end{assumption}

% Suppose the generative model is expressive enough, 

% We propose to find the parameterization function $r$ on the neighborhood of $\boldtheta_t$ by 
% % using this projection. We construct 
% $r(\boldw) = \boldtheta_t + \bolddelta(\boldw)$.   
% $\bolddelta(\boldw)$ is the minimizer of the following objective: 
% \begin{align*}
%      \int_{-\infty}^{\infty} \lambda(\tau) \mathbbE \left[\left(h_{\tau}(X_\tau) - \underbrace{\langle \bolddelta, T(X_\tau) - \mathbbE [T(X_\tau)] \rangle}_{\partial_t (\log q_t)|_{t = \tau} (X_\tau)} \right)^2\right] \mathrm{d}\tau, 
% \end{align*} 
% where $\lambda$ is a positive scalar function which is zero at $\infty$ and $-\infty$. 
\subsection{Matching to NGD}
From \cref{eq.deltasol}, 
we can see the projection $\bolddelta_{t_0}(\boldw)$ creates a connection between the parameter evolution of a generative model and those of a probabilistic model. 
The key idea of this paper is 
to align the evolution of the generative model with that of the probabilistic model. 
In particular, we 
match $\bolddelta_{t_0}(\boldw)$ to the NGD update using the following objective 
\begin{align}
    \label{eq.align}
    \boldw_{t_0} = \argmin_{\boldw \in \mathcal{W}} \| \nabla^N \mathcal{L}(\boldtheta_{t_0}) - \bolddelta_{t_0}(\boldw) \|^2. 
\end{align}
In words, we find the generative model parameter $\boldw$ that results in the projected update $\bolddelta_{t_0}(\boldw)$ closest to the natural gradient update of the loss function. 

In the previous section, we have seen that 
$\bolddelta_{t_0}(\boldw)$ can be approximated using samples $Z$.
% Suppose we have samples from $q_{\boldtheta_{t_0}}$, we can also approximate $\mathcal{F}_{t_0}^{-1}\nabla_{\boldtheta}$. 
Assume that at the starting point, the generator $g(Z, 0; \boldw_{0})$ produces an output distribution $q_{g_0} \in \mathcal{M}(T)$, and the trajectory $q_{g_t}$ is precisely traced by the projected updates, we can expect that the samples generated from $g(Z, t; \boldw_{t})$ will be close to the samples from $q_{\boldtheta_t}$. Thus, we approximate $\nabla^N \mathcal{L}(\boldtheta)$ using samples from $g(Z, t; \boldw_{t})$.

% Notice that we can approximate the above objective using samples from $X_{t_0}$. 

After solving \eqref{eq.align}, instead of taking a natural gradient step, 
we directly sample from $g(Z, t_0 + \epsilon; \boldw_{t_0})$ using a small $\epsilon > 0$. 
Since we have already aligned the projected change of the generative model with the natural gradient step, we can expect that the samples generated from $g(Z, t_0 + \epsilon; \boldw_{t_0})$ will be close to the samples from $q(\boldx; \boldtheta_{t_0} + \epsilon \mathcal{F}_{t_0}^{-1}\nabla \mathcal{L}(\boldtheta_{t_0}))$ by actually taking a natural gradient step with step size $\epsilon$. 
We summarize the entire algorithm in \cref{alg:simple} and name it implicit NGD (iNGD). 

In \cref{fig:illu}, we show an example of the iNGD and compare it with the actual NGD on a Gaussian manifold $T(\boldx) = [\boldx, \boldx\boldx^\top ], \boldx \in \mathbbR^2$. In this example, the generative model is an MLP with one hidden layer consisting of 67 neurons. We can see that the samples generated from iNGD accurately retrace the steps of the classic NGD, ultimately producing a set of samples (red dots) that resemble the target distribution (black dotted line). 

\begin{figure}
    \centering
    \includegraphics[width=.23\textwidth]{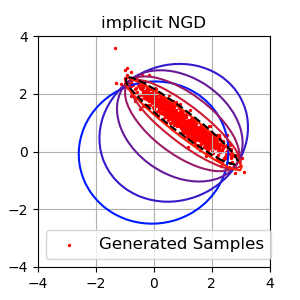}
    \includegraphics[width=.23\textwidth]{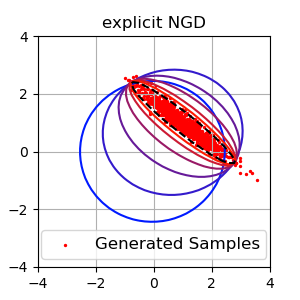}
    \caption{The evolution of parametric distributions under iNGD and NGD on the manifold of Gaussian distributions, i.e., $\mathcal{M}([\boldx, \boldx^\top \boldx])$, starting from $\mathcal{N}(\boldzero,\boldI)$. Each ellipse represents the 95\% confidence interval of $q_{\boldtheta_t}$. As $t$ increases, the ellipse turns from blue to red.   Black dotted line marks the confidence interval the target distribution $p$. For iNGD, the ellipses are approximated by fitting a Gaussian model to the generated samples.
    }
    \label{fig:illu}
\end{figure}

% algorithm
\section{Special Case: Drift Model}
\label{sec.ngd.drift}

% \subsection{Drift Model and Unbiased Estimation}
One popular class of generative models is the \emph{drift-based generative model}. This model iteratively perturbs samples using a vector-valued function until convergence.

\begin{definition}
    \label{ex.drift.model}
    At a fixed time $t_0$, 
    suppose we have samples $X_{t_0}$, 
    a drift generative model generates samples at a new time point $t$ via the following scheme: 
    \[
        \boldg(X_{t_0}, t; \boldw) = X_{t_0} + (t-t_0) \boldh(X_{t_0}; \boldw). 
    \] 
    % This generative model has a ``boundary condition'' that is $\boldg(X_{t_0}, t_0; \boldw) = X_{t_0}$. 
    % where the function $\boldh$ characterizes the updates from random variable $X_{t_0}$. 
\end{definition}
This model could be seen as a local linear model and the amount of update $\boldh(X_{t_0}; \boldw)$ depends linearly on $t - t_0$. 
% Clearly, $\boldg(X_{t_0}, t_0; \boldw) = X_{t_0}$.
% Note that for a drift model, there is no latent variable $Z$. 
The input of the model is a sample at the time point $t_0$. 
% Thus, we cannot generate sample at any given time directly. However, 
To generate samples, we need to draw a batch of samples from an initial distribution $q_{g_0}$ and then successively apply the drift generative model for a sequence of $t$ until convergence.   
An Euler solver of flow-based generative model is an example of a drift generative model, in which case, $t-t_0$ is the step size of the Euler solver.  

% One example of the drift-based model is a discretized  flow-based generative model. The flow 
% \begin{example}
% \label{ex.fine}
%     A fine-tuning generative model $\boldg(Z, t; \boldw)$ given a pretrained generative $\boldg_0(Z; \boldbeta_0)$ is a drift model defined as  
%     \[
%         \boldg(Z, t; \boldw) = \boldg_0(Z; \boldbeta_0) + t \boldw\top \nabla_\boldbeta \boldg_0(Z; \boldbeta_0). 
%     \]
% \end{example}
% We can see this fine-tuning model is a drift model where $X_{t_0}$ is generated by $\boldg_0(Z; \boldbeta_0)$, and $\boldh(X_{t_0}; \boldw)$ is defined as $\boldw\top \nabla_\boldbeta \boldg_0(Z; \boldbeta_0)$. The fine-tuning model can be seen as the first order Taylor expansion of $\boldg_0(Z; \boldbeta)$ at $\boldbeta_0$ and high-order terms are omitted. The omission is justified if the variation $t\boldw = \boldbeta - \boldbeta_0$ is small. 
% One simple example is a perturbation model defined as $\boldx' = \boldx + t\boldh(\boldz) + 
% \boldz$.
% Like in local regression, the bandwidth of the smoothing kernel $\sigma$ in \eqref{eq.tsm} controls the trade-off between bias and variance. 
% Although \eqref{eq.tsm} can be solved for any choice of $\sigma$, 
% Interestingly, the drift generative model enables us to obtain unbiased version of our estimator $\bolddelta_{t_0}(\boldw)$. 

% As we mentioned above, the bandwidth $\sigma$ in \eqref{eq.tsm} introduces a bias to the estimation. 
% One can remove such a bias by letting $\sigma \to 0$.  
We can prove that when using a model introduced in \cref{ex.drift.model}, \cref{eq.deltasol} has a limiting solution as $\sigma \to 0$, which eliminates the bias caused by the smoothing kernel $\lambda$.
% introduced in \cref{eq.locallinear.obj}. 
\begin{theorem}
\label{thm.cf}
    The projection of the time score $s_t$ of a drift generative model has a closed form expression at the limit of $\sigma \to 0$, i.e., 
    \begin{align*}
    \lim_{\sigma \to 0} \bolddelta_{t_0}(\boldw) &= \mathrm{Cov}[T(X_{t_0})]^{-1} \mathbbE \left[ \nabla T(X_{t_0})  \boldh(X_{t_0}; \boldw) \right]\\
    &= \mathcal{F}_{t_0}^{-1} \mathbbE \left[ \nabla T(X_{t_0})  \boldh(X_{t_0} ; \boldw)\right]. 
\end{align*}
\end{theorem}
The proof can be found in \cref{app:kerneldelta}. Using this limiting solution, the objective of \cref{eq.align} can be rewritten as:  
\begin{align}
\label{eq.simple.obj}
    % &\hat{\boldw}_{t_0} \notag \\
    % &= \argmin_{\boldw} 
    \| \nabla \mathcal{L}(\boldtheta_{t_0}) - \mathbbE \left[ \nabla^\top T(X_{t_0})  \boldh(X_{t_0}; \boldw)\right] \|_{\mathcal{F}_{t_0}^{-1}}^2. 
\end{align}
Note that the gradient $\nabla \mathcal{L}(\boldtheta_{t_0})$ is an Euclidean gradient.
In our experiments, this objective function is more stable and computationally efficient than \eqref{eq.align}, since it does not require back-propagating through a matrix inversion. 

\subsection{Kernel NGD}
Now let us consider an example of the drift model where the drift function $\boldh$ is defined as the gradient of an Reproducing Kernel Hilbert Space (RKHS) function.
% Instead of paramaterizing $\boldh$ using a parameter $\boldw \in \mathbbR^{d'}$, we could also consider a non-parametric model as the drift function.  
% $\boldh(\boldx) := \nabla w(\boldx)$, where $w$ is in an RKHS with a kernel function $k$. 

\begin{example}[RKHS Drift Model]
    A kernel drift function is defined as
    $
    \boldh(\boldx; w) := \langle w, \nabla_\boldx k(\cdot, \boldx) \rangle_\mathcal{H}, w \in \mathcal{H}, 
    $
    where $\mathcal{H}$ is an RKHS with a kernel function $k$. 
    % where the function $\boldh$ characterizes the updates from random variable $X_{t_0}$. 
\end{example}

\begin{algorithm}[t]
    \caption{Generative Model Training Guided by the Natural Gradient Descent}
    \label{alg:simple}
    \begin{algorithmic}[1]
    \REQUIRE Target samples $Y \sim p$, latent samples $Z\sim p_Z$, step size $\epsilon$, number of iterations $n$ 
    
    % \STATE Initialize \( \text{sum} \gets 0 \)
    \STATE $t[0] \gets 0$, initialize $\boldw$. 
    \FOR{\( i \gets 1 \) to \( n \)}
        % \STATE \( \text{sum} \gets \text{sum} + A[i] \)
        \STATE Sample $X_{t[i]} \gets g(Z, t[i]; \boldw)$
        \STATE Approximate $\nabla^N \mathcal{L}(\boldtheta_{t[i]})$ with $X_{t[i]}$ and $Y$ using Monte Carlo.  
        \STATE Approximate $\bolddelta_{t[i]}(\boldw)$ with $X_{t[i]}$ using Monte Carlo. 
        \STATE $\boldw \gets \argmin_{\boldw} \| \nabla^N \mathcal{L}(\boldtheta_{t[i]}) - \bolddelta_{t[i]}(\boldw) \|^2$
        \STATE $t[i+1] \gets t[i] + \epsilon$
    \ENDFOR
    
    \STATE \textbf{Return:} Samples from $g(Z, t[n]; \boldw)$
    % \RETURN \( \text{sum} \)
\end{algorithmic}
\end{algorithm}

% In this model, the drift is the \emph{gradient} of $w(\boldx)$. 
To align this generative model with the NGD, 
% In addition to the least squares objective in \cref{eq.simple.obj}, 
we introduce a regularized version of \cref{eq.simple.obj} 
\begin{align}
    \label{eq.simple.obj2}
    \|\nabla \mathcal{L}(\boldtheta_{t_0}) - \mathbb{E}[\nabla T(X_{t_0}) \boldh_w(X_{t_0})]\|_{\mathcal{F}_{t_0}^{-1}}^2 + \lambda \|w\|^2_\mathcal{H}.
\end{align}
 % where $\|\cdot\|_\mathcal{H}$ is the norm in RKHS.  
 % we have the following theorem on the solution to \cref{eq.simple.obj2}.
% and due to the reproducing property, $\langle f, k(\boldx, \cdot)\rangle$
% \eqref{eq.simple.obj} becomes 
% \begin{align*}
%     \label{eq.simple.obj2}
%     &\hat{\boldf}_{t_0} = \argmin_{\boldw} \\
%     &\sum_i \mathbbE \left[ \left(\nabla_i 
%     T(X_{t_0})  f_i(X_{t_0})\right) ^\top \mathcal{F}_{t_0}^{-1}
%      \nabla_i T(X_{t_0}')  f_i(X_{t_0})\right] -\\
%     & \sum_i \nabla^\top \mathcal{L}(\boldtheta_{t_0})  \mathcal{F}_{t_0}^{-1}
%     \mathbbE \left[ \nabla_i T(X_{t_0})  f_i(\boldx)\right] + \mathrm{const.}
% \end{align*}

\begin{theorem}
\label{them.kernelNGD}
 The optimal drift function that minimises \cref{eq.simple.obj2} can be found as
% The unique $w^*$ that minimizes \cref{eq.simple.obj2} 
% gives rise to the following closed form expression for the kernel drift function:
\begin{align}
\label{eq.kernel}
&\boldh_{w^*}(\boldx) =  \mathbbE \left[\nabla T(X_{t_0})  \nabla \nabla  k(X_{t_0},  \boldx)\right]^\top \boldGamma^{-1} \nabla \mathcal{L}(\boldtheta_{t_0}) \\
&\boldGamma =  \lambda \mathcal{F}_{t_0} +
\mathbbE\left[\nabla T(X_{t_0}) \nabla \nabla k(X_{t_0},  X_{t_0}')  \nabla^\top T(X_{t_0}')\right] \notag 
\end{align}
where $X_{t_0}'$ is an independent copy of $X_{t_0}$ and $\nabla \nabla k(\boldx, \boldy)$ is the short hand for $\nabla_\boldx \nabla_\boldy k(\boldx, \boldy)$.     
\end{theorem}
The proof can be found in \cref{sec.proof.kernelNGD}. 
% The update to the particle could be obtained by computing $\boldh^*(\boldx) = \langle w^*, \nabla k(\cdot, \boldx) \rangle$ using the reproducing property of RKHS.  
% \begin{align*}
%     \hat{f}_i (\boldx) = \boldK_i^{-1}
%     \mathbbE \left[ \nabla_i T(X_{t_0})  k_i(X_{t_0}, \boldx)\right] \mathcal{F}_{t_0}^{-1} \nabla \mathcal{L}(\boldtheta_{t_0})  ,
% \end{align*}
% where $\boldK_i = \mathbbE \left[ k_i(X_{t_0}, \boldx) \left(\nabla_i 
%     T(X_{t_0})  \right) ^\top \mathcal{F}_{t_0}^{-1}
%      \nabla_i T(X_{t_0}') k_i(X'_{t_0}, \boldx)\right]$ 
This result enables the \emph{direct calculation} of particle updates without fitting a generative model first.  
This inspires us to build a \emph{particle evolution strategy} guided by NGD: First, we sample $X_0$ from an initial distribution $q_0$, then we iteratively update each sample $X_{t_0}$ using the formula given by \cref{them.kernelNGD} until they converge. This algorithm is summarized in \cref{alg:parNGD} and we name it Kernel implicit NGD (KiNG).

% build a particle based approximation of NGD described in Algorithm \ref{alg:parNGD}.

\begin{algorithm}[t]
    \caption{RKHS/Neural Tangent Kernel Implicit Natural Gradient Descent}
    \label{alg:parNGD}
    \begin{algorithmic}[1]
    \REQUIRE Target samples $Y \sim p$, initial particles $X_0 \sim q_0$, step size $\epsilon$, number of iterations $n$ 
    
    % \STATE Initialize \( \text{sum} \gets 0 \)
    % \STATE Sample $X_{0} \sim q_0$.
    \STATE $t[0] \gets 0$
    \FOR{\( i \gets 1 \) to \( n \)}
        % \STATE $t[i+1] \gets t[i] + \epsilon$
        % \STATE \( \text{sum} \gets \text{sum} + A[i] \)
        % \STATE Sample $X_{t[i]} \gets g(Z, t[i]; \boldw)$
        \STATE Approximate $\boldh$ with $Y$ and $X_{t[i]}$ using \eqref{eq.kernel} or \eqref{eq.ntk.update}.
        \STATE $t[i+1] = t[i] + \epsilon$
        \STATE $X_{t[i+1]} = X_{t[i]} + \epsilon \boldh(X_{t[i]})$
        % \STATE $\hat{\boldw} \gets \argmin_{\boldw} \| \mathcal{F}_{t[i]}^{-1}\nabla_{\boldtheta} \mathcal{L}(\boldtheta_{t[i]}) - \bolddelta(\boldw) \|^2$
    \ENDFOR
    
    \STATE \textbf{Return:} $X_{t[n]} \sim q_{t_{[n]}}$.
    % \RETURN \( \text{sum} \)
\end{algorithmic}
\end{algorithm}

In \cref{fig:kngd}, we plot the trajectories of KiNG with different one-dimensional initial and target distributions. We choose $\mathcal{M}(T)$ to be a Gaussian manifold, i.e., $T(x) = [x, x^2]^\top$. 
In the right plot, the particles do not converge to the bi-modal target distribution, since the movements of our  particles are ``guided'' by the Gaussian manifold, and the best approximation is a Gaussian with a larger standard deviation. This example shows that the particles are indeed guided by the Gaussian manifold throughout the generative process. This phenomenon could be beneficial, if the target is to find the best approximation within a given family or we already have an informative probabilistic model (see \cref{sec.gm}). 

This behaviour can be changed by replacing the Gaussian manifold with a more expressive manifold. In the left plot of \cref{fig:newmanifold}, we run similar experiments by letting $T$ be the Radial Basis Functions (RBFs) and we can see that indeed the particles bifurcate and converge to both modes. 

% \todo{show another plot where particles do converge to gaussian mixture. }

% \todo{Add infinite dimensional kernel NGD}

\subsection{Neural Tangent Kernel NGD}
KiNG can be generalized to other types of kernels.

% Let us consider a different type of drift model where the drift function is the gradient of a neural network with respect to its weights.  
\begin{example}[Neural Tangent Drift Model]
    \label{ex.nt}
    Given a neural network $\boldphi: \mathbbR^{\mathrm{dim}(\boldx)} \to \mathbbR^{\mathrm{dim}(\boldx)}$, a neural tangent drift function is defined as
    \[
    \boldh(\boldx; \boldw) := \nabla_\boldbeta \boldphi(\boldx; \boldbeta_0) \boldw,
    \]
    where $\boldbeta_0$ are initial weights, $\nabla_\boldbeta\boldphi(\boldx; \boldbeta_0)$ is neural tangent. 
    % where the function $\boldh$ characterizes the updates from random variable $X_{t_0}$. 
\end{example}

\begin{figure}[t]
    \centering
    \includegraphics[width=0.47\linewidth]{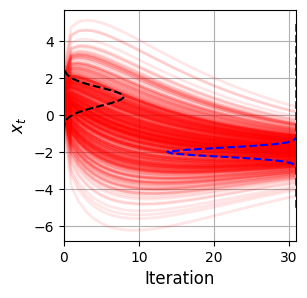}
    \includegraphics[width=0.47\linewidth]{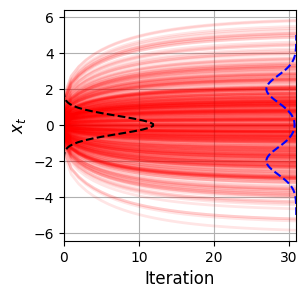}
    \caption{The evolution of particles $X_t$ under KiNG on the manifold of Gaussian distributions. Each red line is a trajectory of a particle in the space-time.  
    The initial distribution $q_0$ is plotted on the left as black dotted lines, while the target distribution $p$ is plotted on the right as blue dotted lines. }
    \label{fig:kngd}
\end{figure}
% \todo{move to experiments section?}
\begin{figure}[t]
    \centering
    \includegraphics[width=0.47\linewidth]{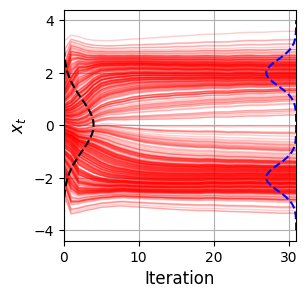}
    \includegraphics[width=0.47\linewidth]{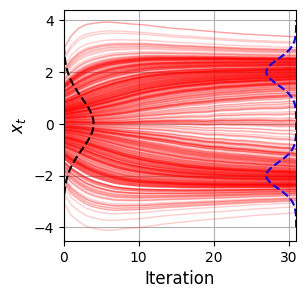}    \caption{Particle trajectories when using more expressive manifold. $T(\boldx) := [k(x_1, b_1), k(x_1, b_2), \cdots]^\top $, where $k$ is an RBF and $b_i$ are kernel basis randomly chosen from samples of $X_{t_0}$. Left: KiNG, Right, ntKiNG. }
    \label{fig:newmanifold}
\end{figure}

Similar to \cref{eq.simple.obj2}, we can solve the following regularized objective to find the update of the particles
\begin{align}
    \label{eq.simple.obj3}
    \|\nabla \mathcal{L}(\boldtheta_{t_0}) - \mathbb{E}[\nabla T(X_{t_0}) \boldh_\boldw(X_{t_0})]\|_{\mathcal{F}_{t_0}^{-1}}^2 + \lambda \|\boldw\|^2.
\end{align}

% When using the Neural Tangent drift model, 
The optimal drift that minimises \cref{eq.simple.obj3} can be expressed using the \emph{neural tangent kernel} \citep{Jacot2018Neural}:
% \begin{align*}
%     \hat{\boldw} &= \boldK^{-1} \mathbbE\left[\nabla_\boldbeta ^\top \boldg_0 \nabla^\top T \right]  \nabla_{\boldtheta} \mathcal{L}(\boldtheta_{t_0}) \\
%     \boldK &:= \mathbbE\left[\nabla_\boldbeta^\top \boldg_0 \nabla^\top T \right]
%     \mathcal{F}^{-1}
%     \mathbbE\left[\nabla T \nabla_\boldbeta \boldg_0 \right]
% \end{align*}
% use Woodbury identity
\begin{align}
    \label{eq.ntk.update}
    &\boldh_{\boldw^*}(\boldx) =   \mathbbE\left[\nabla T(X_{t_0})  \boldK_\mathrm{NTK}(X_{t_0}, \boldx) \right]^\top \boldGamma^{-1} \nabla \mathcal{L}(\boldtheta_{t_0})  \\
    &\boldGamma := \lambda \mathcal{F}_{t_0} +  \mathbbE\left[\nabla T(X_{t_0}) \boldK_\mathrm{NTK}(X_{t_0}, X_{t_0}') \nabla^\top T(X_{t_0}') \right], \notag 
\end{align}
where $\boldK_\mathrm{NTK}$ is the matrix-valued \emph{neural tangent kernel}, defined as $\boldK_\mathrm{NTK}(\boldx, \boldy) := \nabla_\boldbeta \boldphi(\boldx)
\nabla^\top_\boldbeta \boldphi(\boldy)$. 
This result can be proven using the same technique described in \cref{sec.proof.kernelNGD}.
In this paper, we use empirical and a finite-width NTK for simplicity, but NTKs that are infinitely wide can be efficiently computed using off-the-shelf package such as \verb|neural-tangents| \citep{neuraltangents2020} for a variety of neural network architectures. 

The right plot of \cref{fig:newmanifold} shows the particle trajectory of a neural tangent KiNG with $T$ as RBF basis. We name this variant of KiNG as ntKiNG. % This choice is intentional so that we can keep the expressiveness of the sufficient statistic and the generator at the same level to avoid the information collapse mentioned in \cref{sec.proj.change}.

\paragraph{Remark: }
\cref{eq.ntk.update} requires computing a matrix-valued kernel, which may be computationally demanding if the dimensionality of $X_{t_0}$ is high. However, 
in experiments, we observe that the formulation works well with a diagonalized scalar kernel, i.e., 
\begin{align*}
 [\boldK(X_{t_0}, X_{t_0}')]_{l,m \in \left\{1 \dots \mathrm{dim}(\boldx)\right\}} := \begin{cases}
     k(X_{t_0}, X_{t_0}'), & l = m\\
     0, & l \neq m, 
 \end{cases}
\end{align*}
where $k$ is any scalar kernel function. 

\section{Experiments}
\label{sec.exp}
In this section, we further validate our method on toy data and real-world data respectively. The summarized details of experiment setups can be found in \cref{app:expset}. In all experiments, $\mathcal{L}(p, q_\boldtheta) = \mathrm{KL}[p, q_\boldtheta]$. 

\subsection{KiNGs vs. Reverse KL Wasserstein Gradient Flow and MMD Flow}
In this experiment, we compare KiNG, ntKiNG, with reverse KL Wasserstein Gradient Flow (WGF) \citep{gaodeep19,liu2024minimizing} and Maximum Mean Discrepancy flow (MMD flow) \citep{hagemann2024posterior} on small datasets with different dimensions. See \cref{sec.competitor.methods} for more explanations of these methods. 

Since they all minimise different divergences, we measure their performance using MMD \citep{gretton2012akernel} between a fresh batch of target samples and $X_t$.
\begin{figure}
    \centering
    \includegraphics[width=0.47\linewidth]{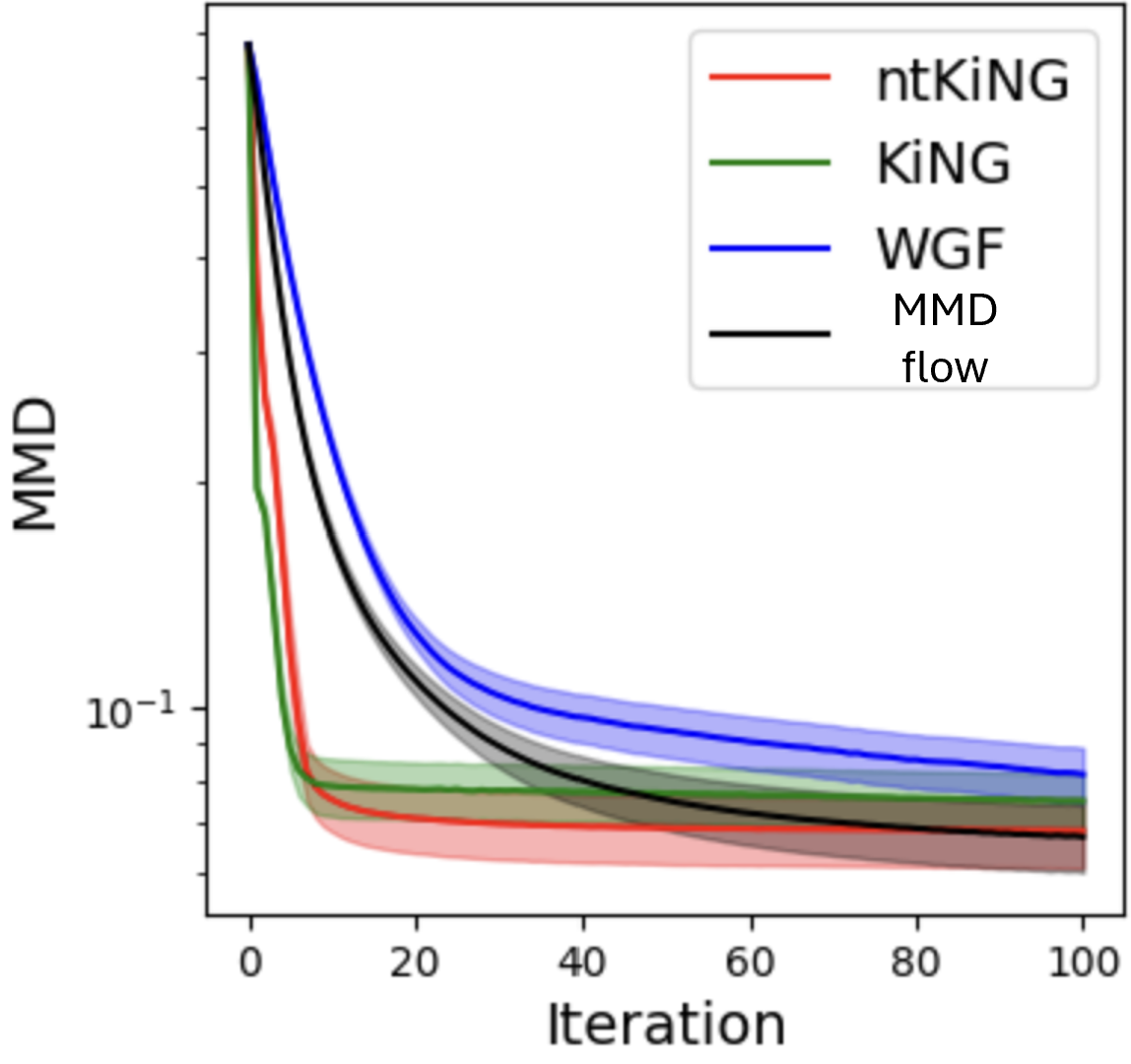}
    \includegraphics[width=0.47\linewidth]{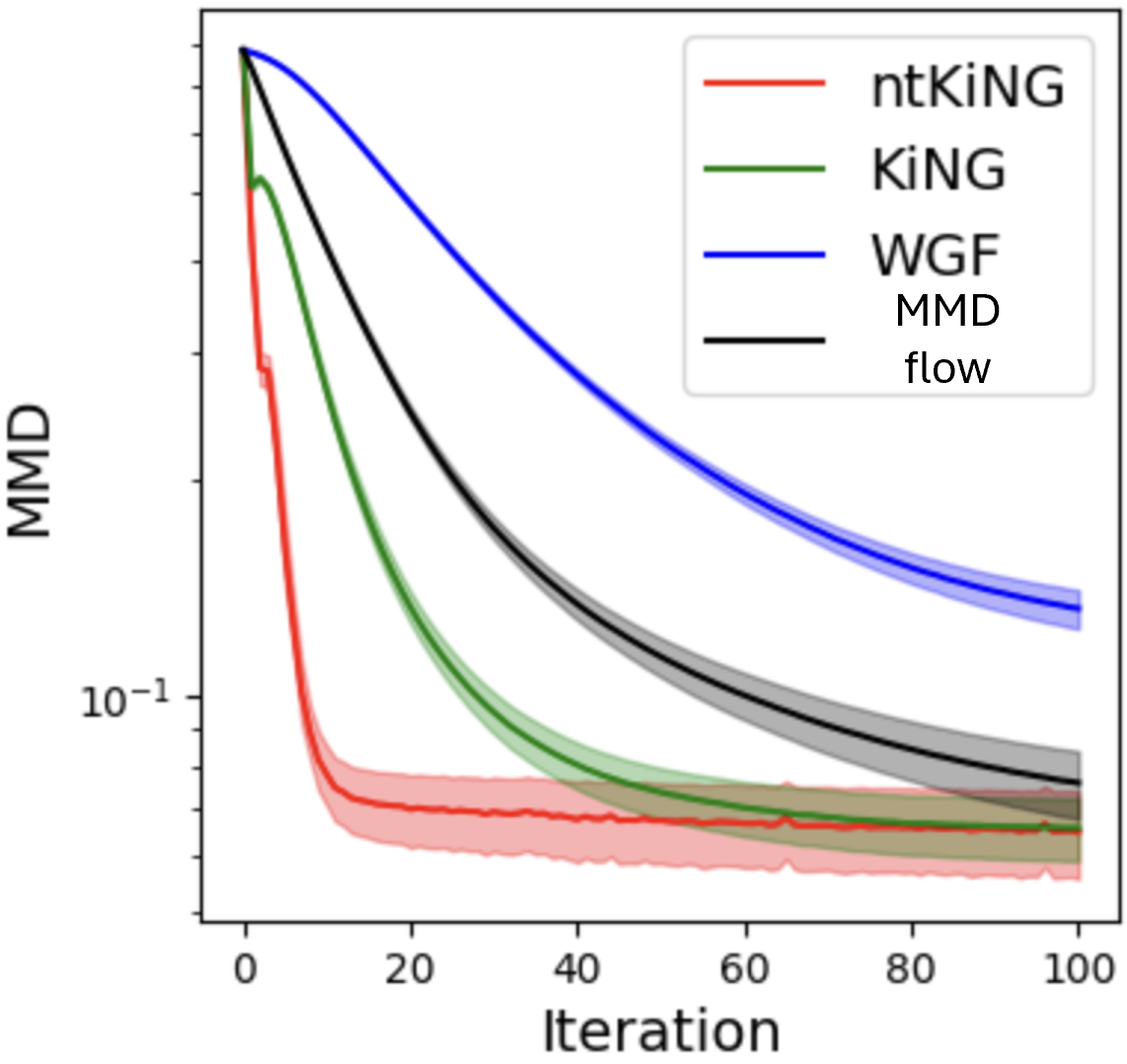}
    \caption{$\mathrm{MMD}[Y, X_t]$ over iterations, the lower the better. Left: 5 dimensions, Right: 20 dimensions. The error bar indicates the standard error. }
    \label{fig:enter-label}
\end{figure}

Let $p = 0.5\mathcal{N}(-2, \boldI) + 0.5\mathcal{N}(2, \boldI)$.
We draw 100 samples from $p$ as $Y$, 100 samples from $\mathcal{N}(0, \boldI)$ as $X_0$, and run KiNG, WGF, MMD flow to evolve particles $X_t$. We plot $\mathrm{MMD}(Y, X_t)$ over iterations. 
For all methods, we set learning rate to be 1, which is the largest learning rates without causing numerical instability. It can be seen that when dimension is small ($5$), all methods work relatively well and ntKiNG and KiNG can reduce MMD faster, but when we increase the dimension to 20 the performance gap widens. However, ntKiNG and KiNG still have a commanding lead. 

\subsection{Graphical Model Recovery with Informative Sufficient Statistics}
\label{sec.gm}
\begin{figure}
    \centering
    \includegraphics[width=0.5\textwidth]{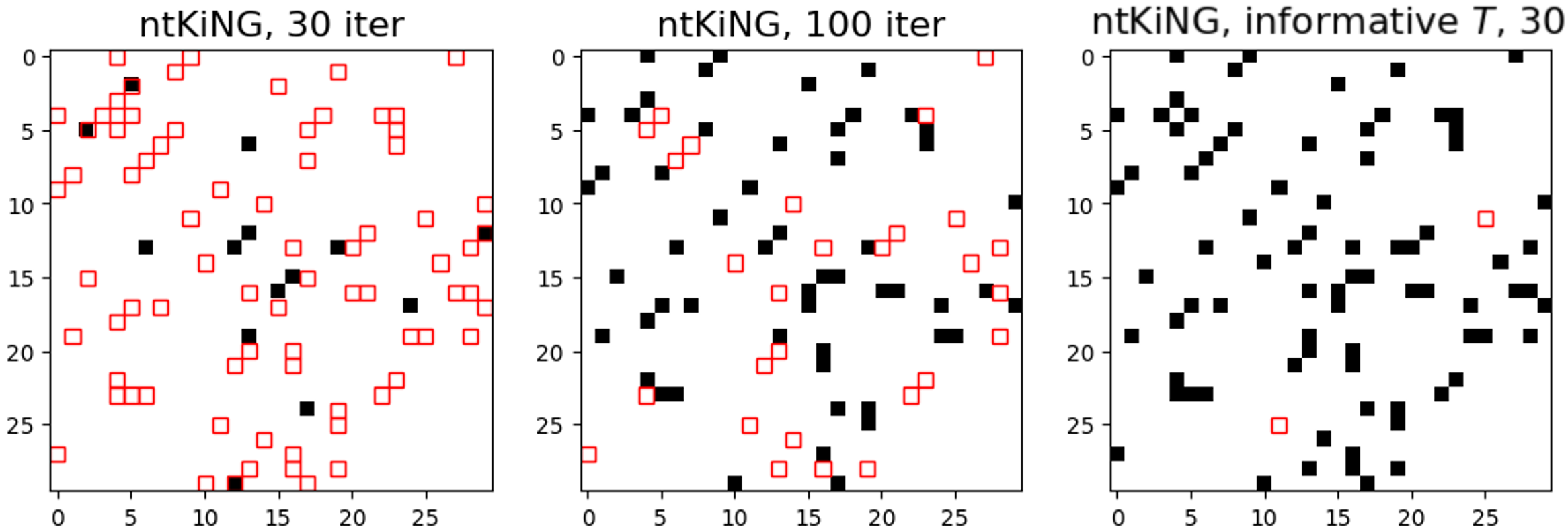}
    \caption{Sparsity pattern of $\boldTheta$ recovered by graphical lasso, using samples trained by ntKiNG with different sufficient statistics and number of iterations (red boxes indicate missing edges, the less boxes the better)}
    \label{fig:graphical_model}
\end{figure}
In this experiment, we showcase how much improvement we can get when using informative sufficient statistics. Exponential family are commonly used to encode graphical models. For example, a Gaussian graphical model is a Gaussian density $p = \mathcal{N}(\boldzero, \boldTheta^{-1})$, where the sparsity pattern of $\boldTheta$ encodes an undirected graph, describing the interactions of the random variables. 
One can imagine that if the generated samples approximate $p$ well, we should recover the correct graphical model from these samples.  
In this experiment, we let $p$ be a 30-dimensional Gaussian graphical model, and draw 200 samples $Y \sim p$, 200 samples $X_0 \sim \mathcal{N}(\boldzero, \boldI)$, and move the $X_0$ toward $p$ using ntKiNG algorithm. Finally we apply the graphical lasso \citep{friedman2007sparse} to estimate graphical models displayed in the left and middle plots of \cref{fig:graphical_model}. Here $T(\boldx) := [\text{RBF basis}]$. 

Since our methods use a probabilistic model to guide its training process, one may wonder if knowing the graphical model structure would improve the performance of the algorithm. To test this, we design a new sufficient statistic $T(\boldx) := [\text{RBF basis}, \forall {(i,j) \in \{(i,j) | \Theta_{i,j} \neq 0} \}, x_i x_j]$, i.e., we added pairwise potential functions that corresponds to pairwise factors in this graphical model. We ran the ntKiNG again with this new, better informed sufficient statistic for 30 runs. The graphical lasso estimate is shown in the right plot of  \cref{fig:graphical_model}. It can be seen that, when using the informed sufficient statistics, ntKiNG could recover the almost-correct graphical structure in only 30 iterations, while it takes the regular ntKiNG much longer. This suggests, our methods can indeed use a pre-existing probabilistic model to accelerate its generative model training process.

\subsection{Covariate Shift by Dist. Matching}
In domain adaptation tasks, samples are drawn from the source distribution $p_{XY}$ and the target distribution $q_{XY}$ where $X, Y$ are covariates and label respectively. The problem is that a classifier trained on source distribution samples may not work on target distribution samples. 
Covariate shift \citep{sugiyama2008direct,quionero2009dataset} refers to a special case where $p_{X} \neq q_{X}$ but $p_{Y|X} = q_{Y|X}$. We adopt a ``marginal transport'' assumption \citep{courty2016optimal} that $X \sim q_X$ are generated as $X = \psi(X')$ 
% \footnote{A more sophisticated setting would allow  $\psi$ to depend on the source label $Y'$ as well, but we don't consider it in this paper. }, 
where $X' \sim p_X$. It means, samples are generated from the source distribution and then ``transported'' to the target domain. For example, images in the source domain contains photos of  objects, while in the target domain, photos contain the same objects but are filtered to reflect certain styles. 

In the covariate shift setting, we observe joint samples from the source $p_{X,Y}$, but only have target domain covariates from $q_X$. The goal is to find $\psi^{-1}$. We find the reverse process by minimising $\mathrm{KL}[p_X, q_t]$ using \cref{alg:parNGD}, where $X_0 \sim q_0$ are set to be the target domain covariates. 

We demonstrate the effectiveness of this algorithm in \cref{fig.transfer.illus}, where the transfer $\psi$ is a clockwise rotation on samples by 45 degrees. An inverse $\psi^{-1}$ is a counter-clockwise rotation and has been correctly identified by ntKiNG.  

We further test our algorithm on the Office+Caltech classification dataset \citep{gong2012geodesic} which is an object recognition dataset with photos collected from four different places: \verb|amazon|, \verb|dslr|, \verb|webcam|, \verb|caltech|. The task is to train a source classifier using one of the places, and test it on samples from another place. We reduce the dimension of the dataset to 50 using PCA and 
test the performance of ntKiNG against two other particle-based transport methods WGF and MMD flow that also matches $q_{t}$ with $p_X$. 
The performance is measured by the percentage gains compared with directly applying the source classifier to the target samples. 
The results in \cref{tab:accuracy_comparison} show that our method achieves the most accuracy gains comparing to WGF and MMD flow. In 10 out of 12 domain adaptations settings, ntKiNG improves the testing accuracy. 

\begin{figure}[t]
    \centering
    \includegraphics[width=.41\textwidth]{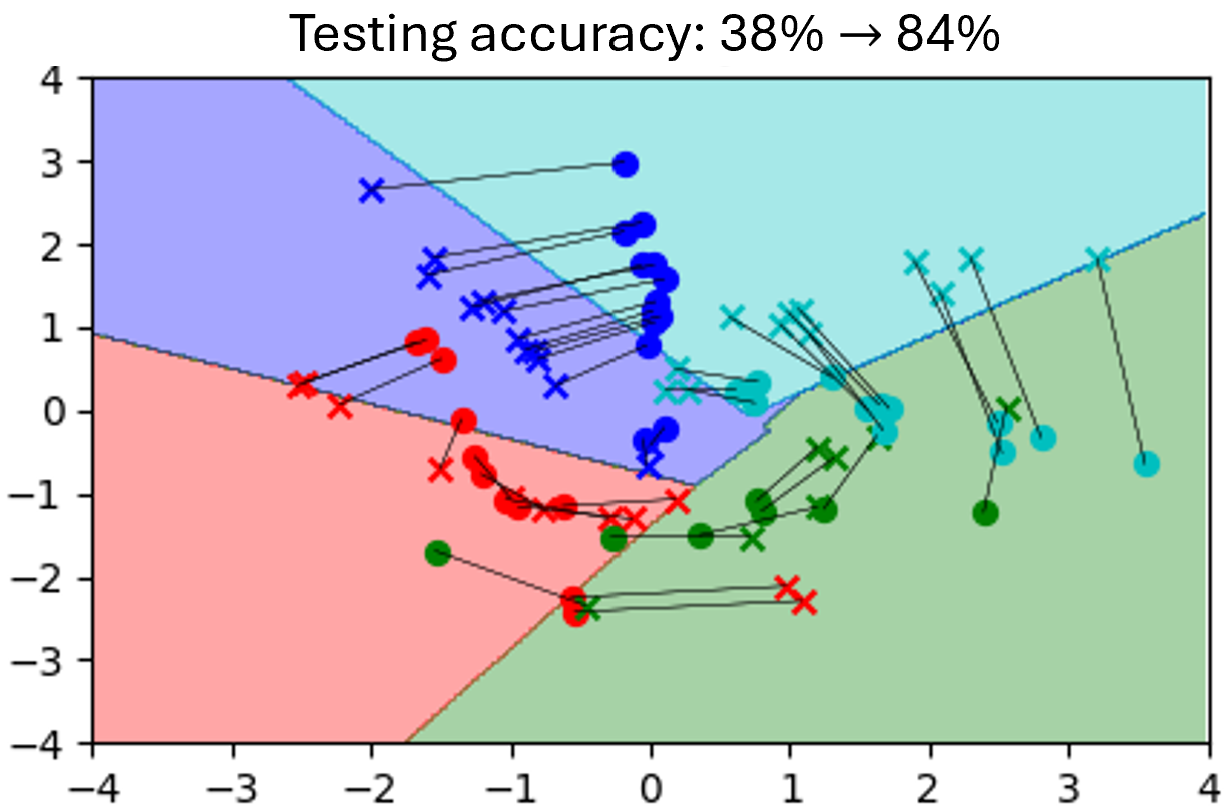}
    \caption{
    The inverse mapping $\psi^{-1}$ found by ntKiNG. The classification boundary of a source classifier trained on \( p \) is depicted in four distinct colors. Target domain samples are marked with \(\bullet\), and KiNG transforms these samples to new positions indicated by \(\times\). Notably, many samples---especially those in the cyan and blue classes---are transported from the incorrect side of the classification boundary to the correct side and the test accuracy increases as a result. }
    \label{fig.transfer.illus}
\end{figure}

\begin{table}[t]
\centering
\begin{tabular}{lcccc}
\hline
Src. $\to$ Tar.        & base ($\%$) & ntKiNG & WGF & MMD flow \\
\hline
amz $\to$ dslr      & 69.50   & \textbf{+13.75}     & -0.50  & +2.75 \\
amz $\to$ web    & 72.75   & \textbf{+8.25}      & -2.00  & +0.75 \\
amz $\to$ cal   & 91.50   & -0.75      & -5.75  & \textbf{+0.00} \\
dslr $\to$ amz      & 86.00   & \textbf{+1.27}      & -6.37  & \textbf{+1.27} \\
dslr $\to$ web      & 98.09   & \textbf{+0.00}      & -0.64  & \textbf{+0.00} \\
dslr $\to$ cal     & 84.08   & \textbf{+5.73}      & -7.64  & +0.64 \\
web $\to$ amz    & 77.97   & \textbf{+3.39}      & -2.71  & +1.02 \\
web $\to$ dslr      & 91.19   & \textbf{+3.39}      & -4.07  & +1.36 \\
web $\to$ cal   & 76.61   & \textbf{+3.39}      & -3.73  & +0.34 \\
cal $\to$ amz   & 82.00   & -2.50      & -4.50  & \textbf{-0.25} \\
cal $\to$ dslr     & 58.25   & \textbf{+16.50}     & +7.00  & +3.25 \\
cal $\to$ web   & 65.50   & \textbf{+10.25}     & +1.00  & +2.00 \\
\hline
\textbf{Average}     & \textbf{79.45}  & \textbf{+5.22} & \textbf{-2.49} & \textbf{+1.09} \\
\hline
\end{tabular}
\caption{Comparison of testing accuracy differences (in $\%$) relative to the base classifier without any transfer learning. }
\label{tab:accuracy_comparison}
\end{table}

\subsection{Sample Denoising with Pretrained Energy-based Model (EBM)}
\begin{figure}[t]
    \centering
    \includegraphics[width=0.99\linewidth]{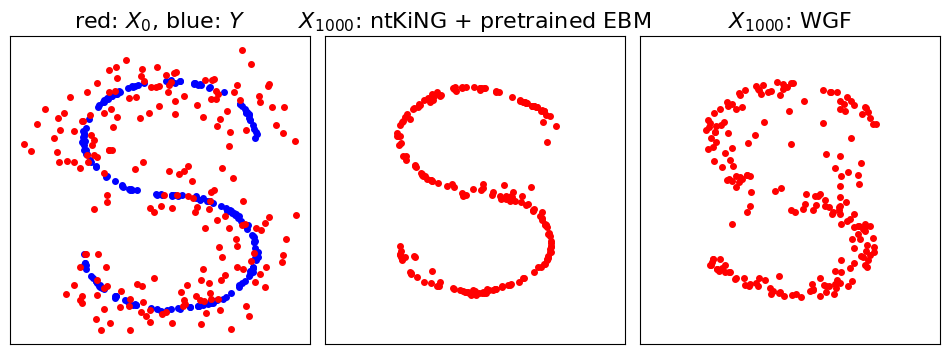}
    \caption{Denoised samples obtained by ntKiNG and WGF. ntKiNG uses pretrained energy-based model as $T$. }
    \label{fig:scurve2}
\end{figure}

\begin{figure}[t]
    \centering
    \includegraphics[width=0.9\linewidth]{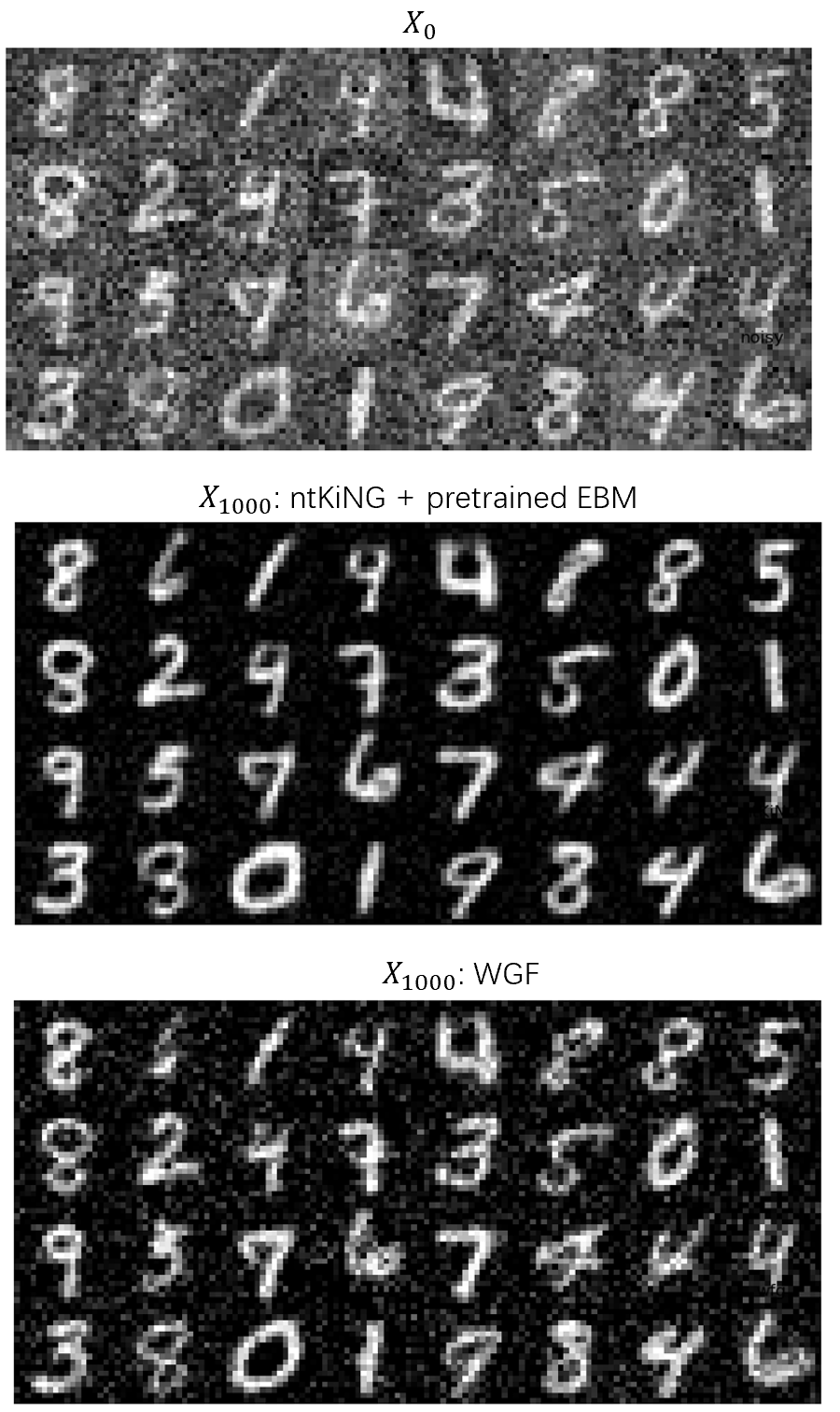}
    \caption{Denoised MNIST images}
    \label{fig:mnist}
\end{figure}

To further investigate the performance of iNGD on higher-dimensional datasets and more sophisticated sufficient statistics, we conduct two denoising experiments. 

We begin by drawing samples $Y\sim q$ and train a deep energy-based model using denoising score matching \citep{song2021train}. We define the sufficient statistic \( T(\boldx) \) as the output of the penultimate layer of the EBM, ensuring that \( T(\boldx) \) captures information about the distribution \( q \). The structure of the EBM can be found in \cref{sec.denoising.ebm}. 

We then construct noisy samples by adding Gaussian noises to $Y$ as $X_0$ and run ntKiNG and WGF using $Y$ and $X_0$ as the target and initial samples with 1000 iterations.  

We first test the procedure on an S-curve dataset and the result is shown in Figure \ref{fig:scurve2}. It can be seen that ntKiNG more effectively recovers the S-shape compared to WGF.

We then perform the experiment on the MNIST dataset (downsampled to 18 by 18), using the same model architecture and estimator to construct \( T \). As we can see from Figure \ref{fig:mnist},
both methods act as decent denoisers, but ntKiNG produces cleaner reconstructions than WGF.

\section{Related Works}

Our methods bridge the gap between generative model training/sampling and parametric model optimization. Both domains are extensively studied in the literature. 

There has been a trend toward using optimization techniques to sample from unknown distributions. These methods first draw samples from an initial distribution and then move them according to a time-dependent velocity field \citep{liu2017stein,Chewi2020,maurais2024sampling}. A typical family of methods is the Wasserstein Gradient Flow \citep{ambrosio2008gradient}, which has found various applications outside of sampling, such as generative modelling \citep{gaodeep19,choi2024scalable} and missing data imputation \citep{chen2024rethinking}. Our method falls within this family of algorithms, and we compared two of its variants in our experiments. 
% However, these methods do not directly work with any probabilistic models. 
To the best of our knowledge, none of the existing approaches could leverage a pre-existing probabilistic model to guide the flow of particles. Our framework is also more general: \cref{alg:simple} works for non-particle based generative models as well (as we see in \cref{sec.gm}). 

Another trend in generative modelling is ``flow matching'', where one aligns the drift function with a pre-constructed flow \citep{lipman2023flow,liu2023flow}. In a similar spirit, our method also aligns the instantaneous change of the generative  distribution with a prescribed dynamics (NGD). However, instead of directly matching the velocity field in the sample space, we match the projections of these changes in the parametric space. This approach avoids building arbitrary "bridges" between the reference and target distributions in sample space and instead leverages an effective parametric optimization algorithm to guide the training of the generative model.

In recent years, there has also been efforts to accelerate and approximate NGD using kernel methods, for example,   \citep{Arbel2020Kernelized,Li2019Affine} propose to approximate the natural gradient by optimizing a dual formulation. 
% It was also noticed that, the samples from the generative model could be also utilized to approximate the Fisher information matrix. 
However, both methods consider optimizing a probabilistic model, rather than a generative model as described in this paper. Performing NGD requires inverting a large matrix. 
Many research on NGD focuses on approximating the inverse the Fisher Information Matrix \citep{martens15optimizing,grosse2016kronecker,george2018fast}. 
Our particle update, e.g., \cref{them.kernelNGD} also requires us inverting a matrix with the dimension of the sufficient statistic. 
It would be an interesting future work to see if these techniques could be adapted to our approach. 

\section{Limitations and Future Works}
% choice of sufficient statistics 
The effectiveness of the guidance heavily relies on the choice of the exponential family manifold, which is determined by the sufficient statistics $T$. If the guidance is weak, the particles may fail to converge to the target distribution, as demonstrated in Figure \ref{fig:kngd}. In this paper, we show that sophisticated sufficient statistics—such as RBF features or a pretrained EBM—can achieve promising results. Developing theories to better understand the choices of sufficient statistics is an important future work.  

% other types of generative models
In this paper, we only focus on drift-based generative model for its simplicity. An interesting future work is studying the applicability of our methods to other types of generative models (e.g., GAN or diffusion model). 
 
% computation
The primary computational bottleneck of our method lies in the inversion of the $\mathrm{dim}(T) \times \mathrm{dim}(T)$ matrix $\boldGamma$ in \eqref{eq.ntk.update}, which become computationally prohibitive if $\mathrm{dim}(T)$ is high. For some choices of $T$ (e.g., RBF), reducing $\mathrm{dim}(T)$ also reduces its expressiveness. Thus, extending iNGD to a high-dimensional $T$ is an urgent future work. Both our method and MMD flow requires computing kernel matrix, which has an $n^2$ computational complexity. However, there is no matrix inversion involved in MMD flow. 

More generally, despite the promising results, the benefits of guiding the generative model in a parametric space remain to be clarified. Developing theories and applications that compare our method with established generative approaches—such as diffusion models—represents an interesting direction for future work.

\section*{Acknowledgements}
We thank four anonymous reviewers and the area chair for their insightful comments. We thank \href{https://sites.google.com/view/sp-monte-carlo/}{Dr. Sam Power}, \href{https://research-information.bris.ac.uk/en/persons/katerina-karoni}{Dr. Katerina Karoni} and \href{https://www.bristolmathsresearch.org/statistical-science/reading-groups/}{Bristol Machine Learning Reading Group} for helpful discussions.

\bibliography{refs}

%%%%%%%%%%%%%%%%%%%%%%%%%%%%%%%%%%%%%%%%%%%%%%%%%%%%%%%%%%%%%%%%%%%%%%%%%%%%%%%
%%%%%%%%%%%%%%%%%%%%%%%%%%%%%%%%%%%%%%%%%%%%%%%%%%%%%%%%%%%%%%%%%%%%%%%%%%%%%%%
% APPENDIX
%%%%%%%%%%%%%%%%%%%%%%%%%%%%%%%%%%%%%%%%%%%%%%%%%%%%%%%%%%%%%%%%%%%%%%%%%%%%%%%
%%%%%%%%%%%%%%%%%%%%%%%%%%%%%%%%%%%%%%%%%%%%%%%%%%%%%%%%%%%%%%%%%%%%%%%%%%%%%%%
\newpage
\appendix
\onecolumn

\section{Proof of Theorem \ref{thm:delta3.1}}
\label{app:3.1}
\begin{theorem}(Theorem 4.1 in \citep{williams2024high}) 
    \cref{eq.locallinear.obj} can be rewritten as the following form
\begin{align}
    & \mathcal{L}(\bolddelta) = \int_0^1 \lambda_{t_0}(t)\mathbb{E}\left[ \langle \bolddelta,T(X_t) - \mathbb{E}[T(X^{\prime}_t)] \rangle^2 \right]\mathrm{d}t + 
    2 \int_0^1 \partial_t \lambda_{t_0}(t) \mathbb{E}\left[ \langle{\bolddelta, T({X_t})} \rangle \right]\mathrm{d}t +\text{const.} 
    \label{eq:deriv:objfinal}
\end{align}
\end{theorem}
Notice that \cref{eq:deriv:objfinal} is a quadratic minimisation problem. 
Let $B_t = T(X_t) - \mathbb{E}[T(X^{\prime}_t)]$. Firstly, we find the derivative of the quadratic term.   
   For each \(t\), since 
   \(\nabla_\bolddelta\,\bigl(\langle\bolddelta,B_t\rangle^{2}\bigr)
   = 2\,\langle \bolddelta,B_t\rangle\,B_t,\)
   it follows that
   \[
   \nabla_\bolddelta
    \int_t \lambda_{t_0}(t)\mathbb{E}\bigl[\,\langle\bolddelta,B_t\rangle^{2}\bigr]\,\mathrm{d}t
   \;=\;
   2\,\int_t \lambda_{t_0}(t)\,\mathbb{E}\bigl[\langle \bolddelta,B_t\rangle\,B_t\bigr]\,\mathrm{d}t.
   \]
   Since \(\mathbb{E}[B_t] = 0,\) one has
   \(\mathbb{E}\bigl[\langle \bolddelta,B_t\rangle\,B_t\bigr] 
     \;=\; \mathrm{Cov}[B_t]\,\bolddelta \;=\;\mathrm{Cov}[T(X_t)]\,\bolddelta.\)\\
 Hence this part becomes
   \[
   2\int_t \lambda_{t_0}(t)\,\mathrm{Cov}[T(X_t)]\,\bolddelta\;\mathrm{d}t.
   \]
Let us differentiate the linear-in-\(\bolddelta\) term.   
   The term
   \(\int_t 2\,\partial_{t}\lambda_{t_0}(t)\,\mathbb{E}\bigl[\langle \bolddelta,\,T(X_t)\rangle\bigr]\mathrm{d}t\)
   is linear in \(\bolddelta\).  Its gradient w.r.t.\ \(\bolddelta\) is simply
   \[
   2\int_t \partial_{t}\lambda_{t_0}(t)\,\mathbb{E}[T(X_t)]\,\mathrm{d}t.
   \]
Putting these together, the gradient of \(\mathcal{L}(\bolddelta)\) is
\[
\nabla_{\bolddelta}\,\mathcal{L}(\bolddelta)
~=~
2\,\int_t \lambda_{t_0}(t)\,\mathrm{Cov}[T(X_t)]\,\bolddelta\,\mathrm{d}t
\;+\;
2\,\int_t \partial_{t}\lambda_{t_0}(t)\,\mathbb{E}[T(X_t)]\,\mathrm{d}t.
\]
To find the minimiser, set this gradient to zero:
\[
0
~=~
2\,\int_t \lambda_{t_0}(t)\,\mathrm{Cov}[T(X_t)]\,\bolddelta\,\mathrm{d}t
\;+\;
2\,\int_t \partial_{t}\lambda_{t_0}(t)\,\mathbb{E}[T(X_t)]\,\mathrm{d}t.
\]
which can be rewritten as
\[
\Bigl(\,\int_t \lambda_{t_0}(t)\,\mathrm{Cov}[T(X_t)]\,\mathrm{d}t\Bigr)\;\bolddelta 
\;=\;
-\;\int_t \partial_{t}\lambda_{t_0}(t)\,\mathbb{E}[T(X_t)]\,\mathrm{d}t.
\]
If the matrix \(\int_0^1 \lambda_{t_0}(t)\,\mathrm{Cov}[T(X_t)]\,\mathrm{d}t\) is invertible—which is precisely the non-degeneracy (invertibility) of the Fisher information—then we can solve uniquely for \(\bolddelta\):

\[
\bolddelta_{t_0}
~=\;
\Bigl(\int_t \lambda_{t_0}(t)\,\mathrm{Cov}[T(X_t)]\,\mathrm{d}t\Bigr)^{-1} 
\Bigl(-\!\!\int_t \partial_{t}\lambda_{t_0}(t)\,\mathbb{E}[T(X_t)]\,\mathrm{d}t\Bigr).
\]
% We can conclude
% \begin{equation}
%     \bolddelta_{t_0}
%   \;=\;
%   \Bigl(\,\int_t \lambda_{t_0}(t)\,\mathrm{Cov}[T(X_t)]\,\mathrm{d}t\Bigr)^{-1}
%   \;\int_t \partial_{t}\lambda_{t_0}(t)\,\mathbb{E}[T(X_t)]\,\mathrm{d}t.
% \end{equation}

\section{Proof of Theorem \ref{thm.cf}}\label{app:kerneldelta}

Recall: 
\begin{align*}
    \bolddelta_{t_0}(\boldw) = - C^{-1} \int_{-\infty}^{\infty} \partial_t\lambda(t, t_0) \mathbbE \left[ T(g(Z, t; \boldw)) \right]\mathrm{d}t, ~~~
    C = \int_{-\infty}^\infty \lambda(t,t_0) \mathrm{Cov}[T(g(Z, t; \boldw))]
    \mathrm{d}t. 
\end{align*}

The first lemma states a few properties of the Gaussian kernel. 
% Under the assumption that $g(Z, t; \bw)$ is a drift model which can be expressed by
% \begin{equation}
%     g(Z, t; \bw)
% \;=\;
% X_{t_0} 
% \;+\;
% (t - t_0)\,\bh\bigl(X_{t_0}; \bw\bigr)
% \end{equation}
\begin{lemma}
\label{lem.kernel}
    $\int_{-\infty}^{\infty} \partial_t \lambda(t, t_0) \mathrm{d} t = 0$. 
    $\int_{-\infty}^{\infty} (t - t_0)\,\partial_t \lambda_\sigma(t,t_0)\,\mathrm{d}t = -1$. 
\end{lemma}
\begin{proof}
    The first result is due to the Fundamental Theorem of Calculus and the fact that $\lambda(t, t_0) \to 0$ as $|t| \to \infty$. Now, we prove the second statement. 
    Since 
\begin{align*}
    \lambda_\sigma(t,t_0) 
\;=\; 
\frac{1}{\sqrt{2\pi\,\sigma^2}}\,
\exp\!\Bigl(\!-\frac{(t - t_0)^2}{2\,\sigma^2}\Bigr), 
\end{align*}
we have 
\begin{align*}
    \int_{-\infty}^{\infty} \lambda_\sigma(t,t_0)\,\mathrm{d}t
\,=\,
1,
\quad
\text{and}
\quad
\partial_t \lambda_\sigma(t,t_0) 
\,=\,
-\frac{(t - t_0)}{\sigma^2}\,\lambda_\sigma(t,t_0).
\end{align*}
We then have with integration by parts
\begin{equation}
\int_{-\infty}^{\infty} (t - t_0)\,\partial_t \lambda_\sigma(t,t_0)\,\mathrm{d}t
\,=\,
\Bigl.(t - t_0)\,\lambda_\sigma(t,t_0)\Bigr|_{-\infty}^{\infty}
\;-\;
\int_{-\infty}^{\infty} \lambda_\sigma(t,t_0)\,\mathrm{d}t
\,=\,
0 \;-\; 1
\,=\,
-1.
\end{equation}
the last equality is due to \(\lim_{|t| \to \infty}(t - t_0)\,\lambda_\sigma(t,t_0) = 0\).
\end{proof}

First, we inspect $\int_{-\infty}^{\infty} \partial_t \lambda(t, t_0) \mathbb{E}\bigl[T\bigl(g(Z,t;\bw)\bigr)\bigr] \mathrm{d}t. $
Using the Taylor expansion on $\mathbb{E}\bigl[T\bigl(g(Z,t;\bw)\bigr)\bigr]$ at \(t_0\), we obtain
% \begin{equation}
%     T\bigl(g(Z,t;\bw)\bigr)
% \,= \,
% T\bigl(X_{t_0}\bigr)
% \;+\;
% (t - t_0)\,\nabla T\bigl(X_{t_0}\bigr)^\top\,\bh\bigl(X_{t_0};\bw\bigr),
% \end{equation}
% and
\begin{equation*}
    \mathbb{E}\bigl[T\bigl(g(Z,t;\bw)\bigr)\bigr]
\,=\,
\mathbb{E}\bigl[T\bigl(X_{t_0}\bigr)\bigr]
\;+\;
(t - t_0)\,\mathbb{E}\bigl[\nabla T\bigl(X_{t_0}\bigr)^\top\,\bh\bigl(X_{t_0};\bw\bigr)\bigr].
\end{equation*}
Note that we don't have higher order terms as the drift model $g(Z, t; \boldw)$ is a linear function of $t$ by definition (see Definition \ref{ex.drift.model}). 

Thus, due to \cref{lem.kernel}, we have
\begin{align}
\label{eq.numerator}
    \int_{-\infty}^{\infty} 
  \partial_t \lambda_\sigma(t, t_0)\,
  \Bigl[
  \mathbb{E}\bigl(T\bigl(X_{t_0}\bigr)\bigr)
  \;+\;
  (t - t_0)\,\mathbb{E}\bigl(\nabla T\bigl(X_{t_0}\bigr)^\top \bh\bigl(X_{t_0};\bw\bigr)\bigr)
  \Bigr]
\,\mathrm{d}t
   & \;=\;
   -\,\mathbb{E}\bigl[\nabla T\bigl(X_{t_0}\bigr)^\top \bh\bigl(X_{t_0};\bw\bigr)\bigr]. %\nonumber\\
   %& = \mathbb{E}\bigl[\nabla T\bigl(X_{t_0}\bigr)\bh\bigl(X_{t_0};\bw\bigr)\bigr]
\end{align}
% the last equality is due to 
% \[
%    \int_{-\infty}^{\infty} (t - t_0)\,\partial_t \lambda_\sigma(t,t_0)\,\mathrm{d}t 
%    = -1.
%    \]
Now we shift our focus on  
\[
C 
\,=\, 
\int_{-\infty}^\infty 
\lambda_\sigma(t,t_0)\,\mathrm{Cov}\bigl[T\bigl(g(Z, t; \bw)\bigr)\bigr]
\,\mathrm{d}t.
\]
As \(\sigma \to 0\), \(\lambda_\sigma(t,t_0)\) converges to \(\delta(t - t_0)\), so $\lim_{\sigma \to 0} \left(\int \lambda(t, t_0) f(t) \mathrm{d}t \right) = f(t_0)$. 
% In the drift model expansion near \(t_0\), 
% \[
% g(Z,t_0;\bw) = X_{t_0},
% \]
% so 
% \[
% T\bigl(g(Z,t_0;\bw)\bigr)
% \,=\, T(X_{t_0}),
% \]
% and thus 
% \[
% \mathrm{Cov}\bigl[T\bigl(g(Z,t_0;\bw)\bigr)\bigr]
% \,=\,
% \mathrm{Cov}\bigl[T(X_{t_0})\bigr].
% \]
Hence 
\begin{align}
\label{eq.lim.C}
    \lim_{\sigma \to 0} C = \lim_{\sigma \to 0} 
\int_{-\infty}^\infty 
\lambda_\sigma(t,t_0)\,\mathrm{Cov}\bigl[T\bigl(g(Z, t; \bw)\bigr)\bigr]
\,\mathrm{d}t
\,=\,
\mathrm{Cov}\bigl[T(X_{t_0})\bigr].
\end{align}
% That is, 
% \begin{align}
% \label{eq.lim.cov}
%     \lim_{\sigma \to 0} C 
% \,=\,
% \mathrm{Cov}[\,T(X_{t_0})\,].
% \end{align}
Finally, combining \cref{eq.numerator} and \cref{eq.lim.C} 
 we have the desired result. 

\section{Proof of Theorem \ref{them.kernelNGD}}
\label{sec.proof.kernelNGD}
\begin{proof}
First, we introduce Welling's Woodbury identity \citep{welling2019notes}:
\[
(P^{-1} + B^T R^{-1} B)^{-1} B^T R^{-1} = P B^T (B P B^T + R)^{-1}.
\]

Recall, that we try to minimise 
    \cref{eq.simple.obj2}  
\begin{align}
    \|\nabla \mathcal{L}(\boldtheta_{t_0}) - \mathbb{E}[\nabla T(X_{t_0}) \nabla \boldh_w(X_{t_0})]\|_{\mathcal{F}_{t_0}^{-1}}^2 + \lambda \|w\|^2_\mathcal{H}. 
\end{align}

Expanding the first square, up to a constant that does not depend on $w$, we obtain 
\begin{align}
\label{eq.king.expanded}
    \mathbb{E}[\nabla T(X_{t_0})  \boldh_w(X_{t_0})]^\top\mathcal{F}_{t_0}^{-1}
\mathbb{E}[\nabla T(X_{t_0}) \boldh_w(X_{t_0})] - 2\nabla^\top \mathcal{L}(\boldtheta_{t_0}) \mathcal{F}_{t_0}^{-1}
\mathbb{E}[ T(X_{t_0}) \boldh_w(X_{t_0})] + \lambda \|w\|^2_\mathcal{H}. 
\end{align}
By definition, $\boldh_w(X_{t_0}) = \langle w, \nabla k(X_{t_0}, \cdot)\rangle$, we obtain a quadratic form with respect to $w$, 
\begin{align}
\label{eq.king.expanded}
    \langle w, \mathbb{E}[\nabla T(X_{t_0})  \nabla k(X_{t_0}, \cdot)]^\top\mathcal{F}_{t_0}^{-1}
\mathbb{E}[\nabla T(X_{t_0}) \nabla k(X_{t_0}, \cdot)]w \rangle  
- \langle w,  2\mathbb{E}[ T(X_{t_0}) \nabla k(X_{t_0}, \cdot)]^\top \mathcal{F}_{t_0}^{-1}
 \nabla \mathcal{L}(\boldtheta_{t_0})\rangle + \lambda \|w\|^2_\mathcal{H}, 
\end{align}
where we used \(
a^\top BCd = (BCd)^\top a = d^\top C^\top B^\top a = \langle d, C^\top B^\top a \rangle
\) on the second term and the fact that the inverse of Fisher Information Matrix $\mathcal{F}_{t_0}$ is a symmetric matrix. 
Differentiating both sides by $w$ and setting the gradient to zero, we obtain the following optimality condition of the least squares: 
\begin{align*}
        2\mathbb{E}[\nabla T(X_{t_0})  \nabla k(X_{t_0}, \cdot)]^\top\mathcal{F}_{t_0}^{-1}
\mathbb{E}[\nabla T(X_{t_0}) \nabla k(X_{t_0}, \cdot)]w 
- 2\mathbb{E}[ T(X_{t_0}) \nabla k(X_{t_0}, \cdot)]^\top \mathcal{F}_{t_0}^{-1}
 \nabla \mathcal{L}(\boldtheta_{t_0}) + 2\lambda w = 0. 
\end{align*}

Thus, the closed form solution of the optimal solution $w^*$ is 
\begin{align}
\left(\mathbb{E}[\nabla T(X_{t_0}) \nabla k(X_{t_0}, \cdot)]^\top\underbrace{\mathcal{F}_{t_0}^{-1}}_{R^{-1}}
\underbrace{\mathbb{E}[\nabla T(X_{t_0}) \nabla k(X_{t_0}, \cdot)]}_{B} 
+ \underbrace{\lambda \boldI}_{P^{-1}} \right)^{-1}\mathbb{E}[\nabla T(X_{t_0}) \nabla k(X_{t_0}, \cdot)]^\top\mathcal{F}_{t_0}^{-1}\nabla \mathcal{L}(\boldtheta_{t_0}). 
\end{align}
Applying Woodbury's identity, 
we get: 
\begin{align}
\mathbb{E}[\nabla T(X_{t_0}) \nabla k(X_{t_0}, \cdot)]^\top \left(
\lambda \mathcal{F}_{t_0} + 
\mathbb{E}[\nabla T(X_{t_0}) \nabla k(X_{t_0}, \cdot)] \mathbb{E}[\nabla T(X_{t_0}) \nabla k(X_{t_0}, \cdot)]^\top
\right)^{-1}\nabla \mathcal{L}(\boldtheta_{t_0}). 
\end{align}
Notice the product $\mathbb{E}[\nabla T(X_{t_0}) \nabla k(X_{t_0}, \cdot)] \mathbb{E}[\nabla T(X_{t_0}) \nabla k(X_{t_0}, \cdot)]^\top = \mathbb{E}[\nabla T(X_{t_0}) \nabla\nabla k(X_{t_0}, X_{t_0}') \nabla^\top T(X_{t_0}')]$, where $\nabla\nabla k(x, y) := \nabla_x\nabla_y k(x, y)$
we obtain the desired result.

\end{proof}

\section{Stein Exponential Family}
 In some applications (e.g. Variational Inference), we do not have samples from the target distribution, instead, we have an unnormalized density. Below, we introduce a special type of exponential family, named Stein exponential family, that allows us to approximate the natural gradient of $\mathrm{KL}[p, q]$.  
% Now, we consider a special parametric exponential family, named Stein exponential family. 

\begin{definition}
    A distribution belongs to the Stein exponential family of a target density $p$ and a test function $\boldf \in \mathbbR^d$ if and only if it belongs to the exponential family with a sufficient statistic that is $T= S_p\boldf$. Given a test function $\boldf = [f_1, f_2, \dots, f_b]$, a Stein operator $S_p\boldf$ of a probability density $p$ is a vector of functions defined as $S_p\boldf = [S_p f_1, S_p f_2, \dots, S_p f_b]$, where 
    \begin{align}
        S_pf_i =  \partial_i(\log p) \boldf + \partial_i \boldf. 
    \end{align}
\end{definition}

An important property of this special type of exponential family is that the expectation of the sufficient statistic is zero under the target distribution, i.e., 
$\mathbbE_{p}[T(\boldx)] = \boldzero$. More discussions on the Stein operator could be found in \citep{}. 
Therefore, if $q_\boldtheta$ is a Stein exponential family distribution, we can write the natural gradient of $\mathrm{KL}[p,q_\boldtheta]$ in \cref{eq.ngd.kl} as
\begin{align}
    \nabla_\boldtheta \mathcal{L}(\boldtheta) = \mathcal{F}^{-1}_t \mathbbE_{q_\boldtheta} \left[S_p f(\boldx)\right].  
\end{align}
Since the Stein operator only requires the gradient of the log density $\log p$ taken with respect to its input, this special design of exponential family enables us to compute the gradient for $\mathrm{KL}[p, q_t]$ using unnormalized $p$ only. 

This parametric model opens up applications such as variational inference and sampling: Given an unnormalized density $p$, we want to train a generative model to sample from $p$. We can design a  Stein exponential family so that we can approximately minimise $\mathrm{KL}[p, q_t]$ using kernel NGD by pushing particles towards the target distribution. 

% \begin{example}
%     Let $\mathcal{H}$ be an RKHS with a kernel function $k$.  
%     A kernel Stein exponential family distribution w.r.t. $\mathcal{H}$ has the following density
%     \begin{align}
%         q(\boldx; \theta) = \exp\left(\langle \theta, S_pk(\boldx, \cdot) \rangle_\mathcal{H} - A(\theta)\right), \theta \in \mathcal{H}. 
%     \end{align} 
% \end{example}

% \begin{definition}
%     A kernel perturbation model w.r.t. $\mathcal{H}^d$ is a time-varying generative model defined as
%     \begin{align}
%         g(\boldz, t; \boldw) = \boldz + t f(\boldz), f \in \mathcal{H}^d. 
%     \end{align}
% \end{definition}

% We propose to approximate the FRGF by projecting the update of $p_t$ onto the parametric space of the exponential family at each time step. 

% \begin{definition} [Projection]
%     The evolution $\partial_t \log p_t$ is projected onto the exponential family parametric space by  
%     \begin{align*}
%         \partial_t \boldtheta^*(t) = \arg\min_{\partial_t \boldtheta(t)} \mathcal{L}_1(\partial_t \log p_t , \partial_t \log q_t ),
%     \end{align*}
%     where $\mathcal{L}$ is a loss function that measures the discrepancy between $\partial_t \log p_t$ and $\partial_t \log q_t$.
% \end{definition}
\section{Reverse KL Wasserstein Gradient Flow and MMD Flow}
\label{sec.competitor.methods}
The WGF dynamics that minimise $\mathrm{KL}[q_t, p]$ move particles $X_t$ using a simple velocity field: 
\begin{align*}
    \frac{\mathrm{d}X_t}{\mathrm{d}t} = \nabla \left(\log p\right)(X_t) - \nabla \left(\log q_t\right)(X_t),  
\end{align*}
where the gradient of log density could be easily estimated via kernel density estimation and the MMD flow minimises $\mathrm{MMD}[Y, X_t]$ using the following velocity field: 
\begin{align*}
    \frac{\mathrm{d}X_t}{\mathrm{d}t} =  N \nabla \left(\frac{1}{N}\sum_{i = 1}^{n}\|X_t - X^{(i)}_t\| - \frac{1}{M}\sum_{j=1}^M\|X_t - Y^{(j)}\|\right). 
\end{align*}

\section{Experiment Setup in Section \ref{sec.exp}}
\label{app:expset}
We summarize the experiments' setup details in this section. For each experiment, we provide details of the dataset and pre-processing procedure, as well as the details of tuning parameters.

\subsection{Comparison with Reverse KL Wasserstein Gradient Flow and MMD Flow}

\subsubsection{Dataset and Pre-processing}
In this experiments, we let $p = 0.5\mathcal{N}(-2, \boldI) + 0.5\mathcal{N}(2, \boldI)$.
We draw 100 samples from $p$ as the target samples $Y$, 100 samples from $\mathcal{N}(0, \boldI)$ as the initial samples $X_0$. No further processing is required. 

\subsubsection{Parameter Tuning }
The main tuning parameter are kernel bandwidth and step sizes. 

For all methods that uses RBF kernel/basis, we set the bandwidth to be the median pairwise distance of all samples. 

For all methods, we use step size 1, as any larger learning rate would result in numerical instability for each method. 

For all methods, we run 100 particle updates. 

The performance metric MMD uses a Gaussian kernel and the bandwidth is set as the median of pairwise distances of all samples $Y$ and $X_t$.  

\subsection{Graphical Model Recovery}

\subsubsection{Dataset and Pre-processing}
In this experiment, we let $p$ be a 30-dimensional Gaussian graphical model, and draw 200 samples $Y \sim p$, 200 samples $X_0 \sim \mathcal{N}(\boldzero, \boldI)$. 
The graphical model $\boldTheta$ is generated as a random graph, with edge probability 0.05. For each non-zero off-diagonal entry, we set $\Theta_{i,j} = 0.3$. No further processing is required. 

\subsubsection{Parameter Tuning }
For all methods, we use the median of sample pairwise distances as the bandwidth. 

For all methods, we set the step size to be 1. 

Parameter tuning of Graphical Lasso is handled by \verb|sklearn| internally using 5-fold cross validation and the sparse graph in graphical model is obtained by truncating all values smaller than 0.1. Below are the Python code. 

\begin{verbatim}
# Fit with cross-validation to select alpha
model_cv = GraphicalLassoCV(alphas=10,  # number of alphas or list of alphas
                            cv=5,       # how many folds in cross-validation
                            max_iter=100, 
                            tol=1e-4)
model_cv = model_cv.fit(x1_test.cpu().numpy())
Theta_cv = model_cv.precision_ 
Theta_cv = Theta_cv > 1e-1
\end{verbatim}

\subsection{Experiment 2: Covariate Shift}
\subsubsection{Dataset and Pre-processing}
We validate our method on the dataset Office+Caltech\footnote{\url{https://github.com/jindongwang/transferlearning/blob/master/data/dataset.md\#office+caltech}}, which is a dataset for domain adaptation, consisting of Office 10 and Caltech 10 datasets. It contains the 10 overlapping categories between the Office dataset and Caltech256 dataset \citep{gong2012geodesic}.

The original features are extracted using a DECAF network, and are 4096 dimensional. 
We apply PCA on the source and target domain to reduce the dimension to 50 with Python code 
\begin{verbatim}
        from sklearn.decomposition import PCA
        pca = PCA(n_components=50)
        pca.fit(X)
        X = pca.transform(X)
        X = X / 100
\end{verbatim}

Due to memory space limit, we also randomly pick 200 samples from all target domains as $X_0$.

\subsubsection{Parameter Tuning }
For all methods, we set step size to 0.1. 

For ntKiNG, we run 100 steps due to reduce the computation cost. 

For WGF and MMD flow, we run 1000 steps. 

The source classifier is an RBF kernel Support Vector Machines with all hyper-parameters chosen by cross-validation with the following python code: 
\begin{verbatim}
    # Split the data into training and test sets (optional)
    X_train, X_test, y_train, y_test = train_test_split(x, y, 
                               test_size=0.3, random_state=42)

    # Define parameter grid
    param_grid = {
        'C': np.logspace(-3, 3, 5),
        'gamma': np.linspace(.2, 5, 5) * gamma,
        'kernel': ['rbf']
    }

    # Create a SVC classifier
    svc = SVC()

    # Initialize GridSearchCV
    grid_search = GridSearchCV(svc, param_grid, refit=True, verbose=2, cv=5)

    # Fit the model
    grid_search.fit(X_train, y_train)
\end{verbatim}
where gamma is the inverse of the median pairwise distances of all inputs. 

\subsection{Experiment 3: Denoising}
\label{sec.denoising.ebm}
The deep EBM is constructed using the following pyTorch code: 

\begin{verbatim}
# Define the MLP-based energy-based model
class EnergyBasedModel(nn.Module):
    def __init__(self, input_dim):
        super(EnergyBasedModel, self).__init__()
        self.network = nn.Sequential(
            nn.Linear(input_dim, 1024),
            nn.SiLU(),
            nn.Linear(1024, 1024),
            nn.SiLU(),
            nn.Linear(1024, 211),
            nn.SiLU(),
            nn.Linear(211, 1)
        )

    def forward(self, x, penultimate=False, flattened=False):
        x = x.view(x.size(0), -1)  # Flatten the input
        if not penultimate:
            return self.network(x)
        else:
            for i, layer in enumerate(self.network):
                x = layer(x)
                if i == len(self.network) - 2:
                    break
            return x
\end{verbatim}

The model is trained using 10000 samples of $Y$, with a batch size 777, and adam optimizer with a step size of 0.001. 
In the S-curve experiment, we add a Gaussian noise to the target sample $Y$ with a standard deviation 0.3, and in the MNIST experiment, we add a Gaussian noise to $Y$ with a standard deviation 0.2. 

The noise used in denoising score matching mathces the noise added to $Y$.

\end{document}